\documentclass{article}

\usepackage[preprint]{style}

\usepackage[utf8]{inputenc} %
\usepackage[T1]{fontenc}    %
\usepackage[pagebackref=true, hidelinks]{hyperref}
\usepackage{url}            %
\usepackage{booktabs}       %
\usepackage{amsfonts}       %
\usepackage{nicefrac}       %
\usepackage{microtype}      %
\usepackage{xcolor}         %
\usepackage{svg}
\usepackage{cancel}
\usepackage{mathtools}

\usepackage{amssymb,amsmath,amsthm}
\usepackage[noabbrev,nameinlink,capitalise,compress]{cleveref}
\usepackage{mathtools}
\usepackage{resizegather}
\usepackage{subcaption}
\usepackage{makecell}

\renewcommand*{\backrefalt}[4]{\ifcase #1 No citations. \or Cited on page #2. \else Cited on pages #2. \fi}

\newcommand{\norm}[1]{\Vert {#1} \Vert}

\newtheorem{theorem}{Theorem}[section]
\newtheorem{proposition}{Proposition}[theorem]

\newcommand{\R}{\mathbb{R}}

\newcommand{\rms}{\mathrm{RMS}}
\newcommand{\softmax}{\text{softmax}}

\theoremstyle{definition}
\newtheorem{definition}{Definition}[section]

\definecolor{phil_green}{rgb}{0.1,0.8,0.2}
\definecolor{laker_blue}{rgb}{0.1,0.2,0.8}
\definecolor{pres_red}{rgb}{0.93, 0.3, 0.17}
\definecolor{leloy_orange}{rgb}{0.9, 0.5, 0}

\title{Training Transformers\\with Enforced Lipschitz Bounds}

\makeatletter
\renewcommand{\@fnsymbol}[1]{\ifcase#1\or \textbf{*}\or 1\or 2 \or 3\fi}
\makeatother

\author{
  Laker Newhouse\thanks{Equal contribution. Correspondence to \texttt{lakern@mit.edu}}
  \\MIT CSAIL
    \And
  R. Preston Hess\footnotemark[1] %
  \\MIT BCS
    \And
  Franz Cesista\footnotemark[1] %
  \\Independent
    \AND
  Andrii Zahorodnii
  \\MIT BCS
    \And
  Jeremy Bernstein
  \\MIT CSAIL
    \And
  Phillip Isola
  \\MIT CSAIL
}

\begin{document}

\maketitle

\begin{abstract}

Neural networks are often highly sensitive to input and weight perturbations. This sensitivity has been linked to pathologies such as vulnerability to adversarial examples, divergent training, and overfitting. To combat these problems, past research has looked at building neural networks entirely from Lipschitz components. However, these techniques have not matured to the point where researchers have trained a modern architecture such as a transformer with a Lipschitz certificate enforced beyond initialization. To explore this gap, we begin by developing and benchmarking novel, computationally-efficient tools for maintaining norm-constrained weight matrices. Applying these tools, we are able to train transformer models with Lipschitz bounds enforced throughout training. We find that optimizer dynamics matter: switching from AdamW to Muon improves standard methods---weight decay and spectral normalization---allowing models to reach equal performance with a lower Lipschitz bound. Inspired by Muon's update having a fixed spectral norm, we co-design a weight constraint method that improves the Lipschitz vs. performance tradeoff on MLPs and 2M parameter transformers. Our \textless 2-Lipschitz transformer on Shakespeare text reaches validation accuracy 60\%. Scaling to 145M parameters, our \textless 10-Lipschitz transformer reaches 21\% accuracy on internet text. However, to match the NanoGPT baseline validation accuracy of 39.4\%, our Lipschitz upper bound increases to $10^{264}$. Nonetheless, our Lipschitz transformers train without stability measures such as layer norm, QK norm, and logit tanh softcapping.

\end{abstract}

\section{Introduction}

Lipschitz bounds for neural networks---bounds on the sensitivity of the model to input perturbations---are of interest for their effect on generalization and robustness \citep{bartlett2017spectrallynormalizedmarginboundsneural,tsuzuku2018lipschitzmargintrainingscalablecertification} and for applications such as differential privacy \citep{bethune2024dpsgdclippinglipschitzneural}. Seminal work \citep{arjovsky2016unitaryevolutionrecurrentneural,cisse2017parseval,yoshida2017spectralnormregularizationimproving,anil2019sorting} enforces Lipschitz bounds beyond initialization for MLPs, RNNs, and GANs, but for transformers, the closest related work, LipsFormer \citep{qi2023lipsformer}, does not constrain the weight matrices during training. Without constraints, large-scale transformer training may encounter instabilities, which has been attributed to attention and output logits growing too large \citep{wortsman2023smallscaleproxieslargescaletransformer,dehghani2023scalingvisiontransformers22}. Can enforced Lipschitz bounds benefit transformers, too? Specifically, we ask:

\begin{center}
    \textit{Can transformers with small, enforced Lipschitz bounds perform well?\\
    How does the weight constraint method affect the Lipschitz versus performance tradeoff?}
\end{center}

\begin{figure}[t]
  \centering
  \includegraphics[width=1\linewidth]{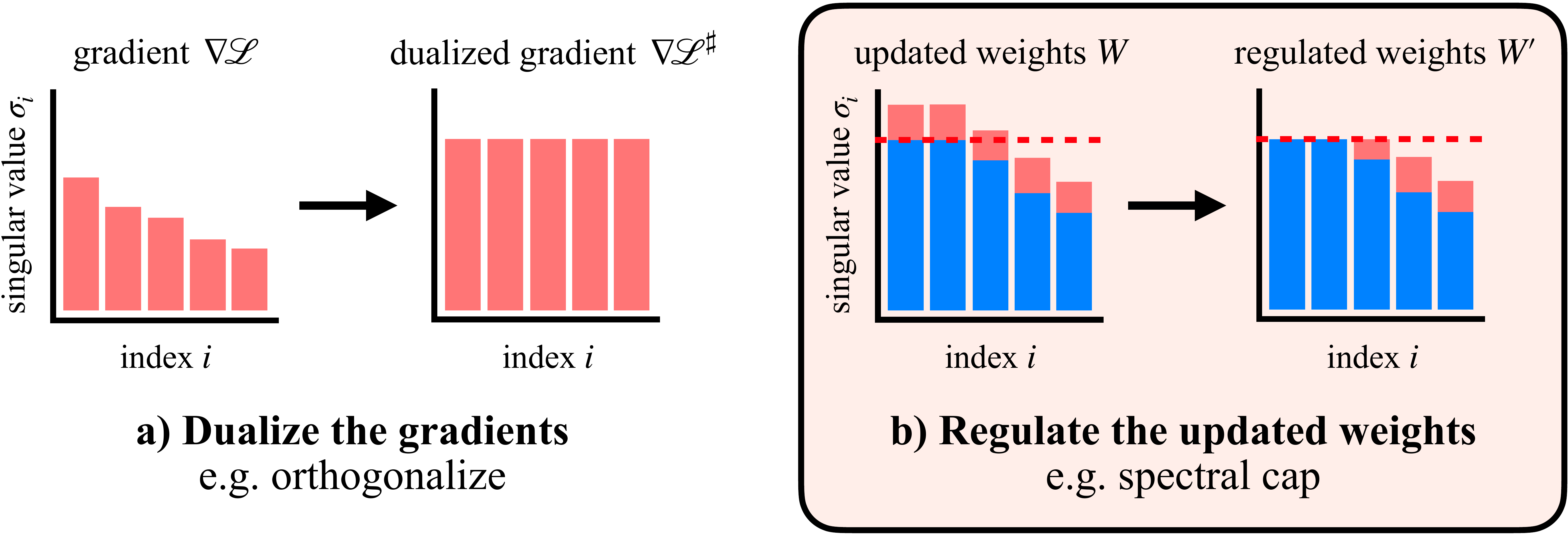}
  \caption{\textbf{To train fast and stably, regulate the gradients and also the weights.} Current research is considering efficient computational primitives for spectrally regulating gradient updates, e.g. Muon \citep{jordan2024muon}. There is an opportunity to construct similar techniques to spectrally regulate the weights themselves. While weight decay with parameter $\lambda$ does implicitly contract singular values by a fixed proportion $U \Sigma V^\top \mapsto U (1 - \lambda) \Sigma V^\top$, there is a broader family of possible singular value maps $U \Sigma V^\top \mapsto U f(\Sigma) V^\top$ for any nonnegative elementwise function $f$. For example, in this paper, we introduce \textit{spectral cap}, a method to efficiently clip singular values above a threshold $\sigma_\text{max} > 0$.} %
  \label{fig:pareto}
\end{figure}

Large scale training often suffers from instabilities such as exploding attention logits, bandaged over by methods such as QK norm \citep{henry2020querykeynormalizationtransformers,dehghani2023scalingvisiontransformers22} or the recent MuonClip optimizer for 1T parameter scale training \citep{kimik2_2025}. We hypothesize that, at this scale, enforcing a Lipschitz bound may offer a direct resolution to loss spikes.

Enforcing Lipschitz bounds on a transformer is challenging because transformers include components that are not globally Lipschitz, such as self-attention \citep{kim2021lipschitzconstantselfattention}. We build on \citet{large2024}'s work, which, similar to LipsFormer, enables Lipschitz continuity by reparameterizing residual connections and modifying self-attention; however, the full story is elusive. LipsFormer goes further than \citet{large2024} to eliminate layer norm \citep{ba2016layernormalization}, but the \href{https://github.com/IDEA-Research/LipsFormer/blob/main/models/lipsformer_swin.py#L205-L206}{official implementation} may make Lipschitz bounds impossible by setting $\epsilon = 0$ in QK norm. In contrast, we remove activation normalization to explore whether training can proceed with no stability measures.

To develop a toolkit for training transformers with an enforced Lipschitz bound, in \cref{sec:comparing-methods} we compare several methods for constraining weight norm. Surprisingly, we find that optimizer choice matters: standard methods such as weight decay \citep{krogh1991simple} and spectral normalization \citep{yoshida2017spectralnormregularizationimproving} improve a Lipschitz versus performance tradeoff more with Muon \citep{jordan2024muon} than with AdamW \citep{loshchilov2019decoupledweightdecayregularization}. We see improvement under Muon for MLPs trained on CIFAR-10 \citep{krizhevsky2009cifar} and corroborate this improvement on 2M parameter transformers trained on Shakespeare text \citep{karpathy_nanoGPT}.

Beyond standard methods, we are inspired by a property of the Muon optimizer---its weight updates have small, known spectral norm---to design a weight constraint method called \textit{spectral soft cap}, which enforces a desired maximum spectral norm $\sigma_\text{max}$ by approximating the map $\sigma \mapsto \min(\sigma_\text{max}, \sigma)$ on all singular values $\sigma$ in parallel by iterating odd polynomials on the weights. \cref{thm:spectral-soft-cap-bound} proves that spectrally capping the singular values bounds weight norm when training with Muon; we provide no theoretical guarantee for AdamW because the spectral norm of its update is not controlled. For AdamW, we explore a second technique that may be better suited to low stable rank updates, although we do not provide provable guarantees. At every step, this technique finds the largest weight singular value and sets it to $\sigma_\text{max}$. In analogy with a hammer that strikes the nail that sticks out the most, we call this technique \textit{spectral hammer}. Our experiments suggest that the most effective combination is Muon with spectral normalization or spectral soft cap, while for the Adam optimizer the only technique that elicits a competitive performance versus Lipschitz tradeoff is spectral hammer.

In \cref{sec:enforcing-for-transformers}, we scale up enforced weight constraint methods to the NanoGPT speedrun benchmark \citep{modded_nanogpt_2024}, training 145M parameter transformers to competitive performance without layer norm or QK norm. We train a $10$-Lipschitz transformer to 21.2\% validation accuracy, compared to the non-Lipschitz baseline of 39.4\% validation accuracy. However, to reach a competitive accuracy of 39.4\%, our global Lipschitz upper bound becomes astronomical at $10^{264}$. While \citet{fazlyab2023efficientaccurateestimationlipschitz} describe ways to tighten Lipschitz bounds, inspecting the maximum activation norms reveals that our model operates far from the worst case. On a particular batch of 393K tokens, the non-Lipschitz baseline has maximum activation entry 148,480 while the $10^{264}$-Lipschitz transformer has maximum activation entry 160. Empirically small activations in Lipschitz-constrained transformers may present an opportunity for low-precision training and inference.

Our code is available at \url{https://github.com/Arongil/lipschitz-transformers}, and our data is available at \url{https://huggingface.co/phess2/lipschitz-transformers}.

In summary, \textbf{our contributions are as follows:}
\begin{itemize}
    \item We train transformers with enforced Lipschitz constraints up to 145M parameters---demonstrating feasibility for full weight matrix constraints in transformer training---including a \textless 10-Lipschitz transformer that achieves 21\% accuracy on FineWeb10B internet text and a \textless 2-Lipschitz transformer that achieves 60\% accuracy on Shakespeare text.
    \item We present evidence that weight decay and spectral normalization yield greater benefits when trained with the Muon optimizer compared to AdamW, matching accuracy with smaller Lipschitz bound. We verify standard robustness properties hold when training with Muon.
    \item We introduce two weight norm constraint techniques: \textit{spectral soft cap} and \textit{spectral hammer}. Out of weight regularization methods for AdamW, spectral hammer elicits the most competitive Lipschitz-constrained performance. For Muon, we prove spectral soft cap bounds weight norm and find that it performs similarly or slightly better than spectral normalization. 
\end{itemize}

\section{Related work}
\label{sec:related-work}

There is now a large literature on the many potential and realized benefits of Lipschitz neural networks \citep{bethune2022pay,bethune2024deep,pmlr-v137-rosca20a}. Input--output Lipschitz certificates may be useful in deployment scenarios where there is a benefit to having strong robustness with respect to input perturbations. Examples include robotic control \citep{neural-fly,Wang2019VerificationON}, classification in the presence of adversarial input perturbations \citep{szegedy2013intriguing}, and in protocols for AI safety \citep{brown-cohen2024scalabale}. Input--output Lipschitz certificates are also used in certain generalization guarantees for deep networks \citep{bartlett2017spectrallynormalizedmarginboundsneural,neyshabur2018a,dherin2022why}.

Similarly, weight--output Lipschitz certificates may be useful in situations where there is an interest in perturbing the weights of a neural network without incurring unstable output behaviour. A prime example is for stable \citep{qi2023lipsformer,flynn2017duality} and scalable \citep{large2024} training. The recent MuonClip optimizer \citep{kimik2_2025} similarly addresses exploding attention logits by constraining query and key weights directly; weight norm constraints are the subject of \cref{sec:comparing-methods}. A second example is for the design of differentially private training algorithms where there is a need to add carefully calibrated noise to the network weights \citep{bethune2024deep,bethune2024dpsgdclippinglipschitzneural}. And a third example is to aid in weight quantization \citep{pmlr-v119-elthakeb20a, Weng_Zhao_Liu_Chen_Lin_Daniel_2020}. 

In fact, there is a close theoretical link between input--output Lipschitzness and weight--output Lipschitzness in deep learning \citep{large2024, bethune2024deep}. The reason is that when we compose two subnetworks, Lipschitzness with respect to the weights of the first subnetwork depends upon the degree of input--output Lipschitzness of the second subnetwork.

Various techniques have been proposed for producing neural networks amenable to Lipschitz certification. These include techniques for modifying the network architecture and training process to improve the resulting Lipschitz properties. For example, spectral normalization \citep{miyato2018spectralnormalizationgenerativeadversarial,Gogianu2021SpectralNF} has been proposed as a means to control the Lipschitz properties of individual weight matrices. New nonlinearities \citep{anil2019sorting} and normalization layers \citep{qi2023lipsformer} have also been proposed. 

Furthermore, given a trained model of a given architecture, various techniques have been proposed for deriving Lipschitz certificates. Deriving the exact Lipschitz constant (i.e. the least upper bound) is known to be computationally hard \citep{katz,NEURIPS2018_d54e99a6,pmlr-v80-weng18a} so researchers settle for producing slacker upper bounds. One approach to producing upper bounds---used in this paper for simplicity---is to obtain Lipschitz statements for each component in the architecture and add or multiply them as appropriate to compose them \citep{szegedy2013intriguing}. However, tighter approaches have also been proposed: \citet{pmlr-v80-weng18a,pmlr-v97-weng19a} provide examples.

\section{Weight norm constraints to enforce Lipschitz constraints}
\label{sec:comparing-methods}

A function $f(x)$ has Lipschitz bound $K$ under a norm $\norm{\cdot}$ if it satisfies $\norm{f(x_1) - f(x_2)} \leq K \cdot \norm{x_1 - x_2}$ for all inputs $x_1, x_2$. The Lipschitz constant is the smallest Lipschitz bound. For neural networks, the most common operation is matrix multiplication which has $\ell_2$ Lipschitz constant equal to the spectral norm of the weight matrix. Constraining the spectral norm of weight matrices is not new, with past work primarily exploring weight decay, spectral normalization, and orthogonal constraints \citep{krogh1991simple,yoshida2017spectralnormregularizationimproving,miyato2018spectralnormalizationgenerativeadversarial,gouk2020regularisationneuralnetworksenforcing,kexuefm-spectral-weight-decay}. These methods have been tested with the AdamW optimizer and have shown benefits for generalization and adversarial robustness \citep{bartlett2017spectrallynormalizedmarginboundsneural,tsuzuku2018lipschitzmargintrainingscalablecertification}. The Muon optimizer introduces new possibilites by ensuring small, fixed-norm weight updates. Inspired by this property, we revisit existing methods and develop new methods for constraining weights. We ask the question:

\begin{center}
    \textit{What is the best way to enforce weight norm constraints throughout training?}
\end{center}

We compare seven methods based on how well they 1) maintain high performance, 2) enforce weight norm constraints, and 3) trade off performance with a Lipschitz bound. To summarize our conclusions, we find that the Muon optimizer achieves lower Lipschitz bounds and better performance compared to AdamW. Among the constraint methods, our experiments suggest that spectral soft cap, spectral hard cap, and spectral normalization meet these criteria best.

\textbf{Muon enables hard weight constraints.} Unlike in AdamW, the weight update norm in Muon is bounded by the learning rate---if its orthogonalizing polynomial never exceeds 1. We follow \citep{You2025} to ensure this property in our experiments. \citet{pethick2025trainingdeeplearningmodels} noted that bounded weight update spectral norm upgrades weight decay with parameter $\lambda$ to enforce a \textit{strict} spectral norm constraint of $1/\lambda$. The reason is that an equilibrium occurs between the update step and weight decay when the weight norm $w$ satisfies $w = w(1 - \lambda \eta) + \eta$ for learning rate $\eta > 0$. See \cref{app:spectral-softcap} for details. We hypothesize that this property may explain our evidence that Muon, compared to AdamW, improves the Lipschitz vs. performance tradeoff for standard methods such as weight decay.

\textbf{A spectral generalization of weight decay.} Weight decay can be viewed as a special case of an \textit{odd polynomial iteration} applied to the weights. Odd polynomials are special because they act directly on the singular values: $p(U \Sigma V^\top) = U p(\Sigma) V^\top$, where $U\Sigma V^\top$ is a singular value decomposition. The odd polynomial for weight decay is $p(x) = (1 - \eta \lambda) x$, where $\eta$ is the learning rate and $\lambda$ is the weight decay. \citet{cisse2017parseval} explored an orthogonalizing polynomial $p(x) = (1 + \beta)x - \beta x^3$, but \citet{miyato2018spectralnormalizationgenerativeadversarial} note that pressuring all singular values toward one limits information in the spectrum. Their method, \textit{spectral normalization}, enforces norm constraints while allowing singular values less than 1 \citep{gouk2020regularisationneuralnetworksenforcing}, but normalization accomplishes the constraint by scaling down the entire spectrum. This global effect motivates a more targeted approach
: penalizing only the singular values that are too large, leaving smaller ones untouched. For a desired maximum norm $\sigma_\text{max} \geq 0$, an idealized penalty would apply $\min(\sigma_\text{max}, \sigma)$ to the singular values, but exactly computing the SVD is slow. Odd polynomial iterations serve as a fast and effective approximation. We contribute a family of such approximations called \textit{spectral soft cap} that contains weight decay as a special case. The derivation and discussion is in \cref{app:spectral-softcap}.

\subsection{Methods for controlling weight norm}\label{sec:methods_weight_constraints}

We are interested in controlling the $\rms \to \rms$ operator norm---a rescaled spectral norm---which has emerged as natural for deep learning \citep{yang2024spectralconditionfeaturelearning,bernstein2024modulardualitydeeplearning}. Unit $\rms\to\rms$ norm is equivalent to a spectral norm of $\sqrt{d_\mathsf{out}/d_\mathsf{in}}$ for a weight matrix $W \in \R^{d_\mathsf{out} \times d_\mathsf{in}}$. In what follows, we denote the principal singular vector subspace with singular value $\sigma_1 \geq 0$ by $\sigma_1 u_1 v_1^\top$, computed via power iteration. We briefly review some known methods to constrain weight norm, then introduce two new methods called spectral capping and spectral hammer.

\textbf{Weight decay}, or Frobenius norm regularization, maps $W \mapsto (1 - \lambda \eta)W$ where $\lambda > 0$ is the decay parameter and $\eta > 0$ is the learning rate, guaranteeing a norm bound in conjunction with Muon.

\textbf{Spectral weight decay}, or spectral norm regularization, targets only the top singular value, mapping $W \mapsto W - \lambda \sigma_1 u_1 v_1^\top$ where $\lambda > 0$ is the decay parameter \citep{yoshida2017spectralnormregularizationimproving,kexuefm-spectral-weight-decay}.

\textbf{Spectral normalization}, originally introduced in GAN training \citep{miyato2018spectralnormalizationgenerativeadversarial}, guarantees a spectral norm bound by mapping $W \mapsto \tfrac{W}{\sigma_1}$.

\textbf{Stiefel manifold projection} pressures all singular values toward 1 using an odd polynomial iteration $W \mapsto p(W)$, leaving open the choice of polynomial $p$. We follow \citet{You2025} whose polynomial converges very rapidly rather than the polynomial from \citep{cisse2017parseval}. Although Stiefel manifold projections usually refers to projections for rectangular matrices, with a slight abuse of notation we use it to describe this operation on both square and rectangular matrices.

\begin{figure}[t]
  \centering
  \includegraphics[width=\linewidth]{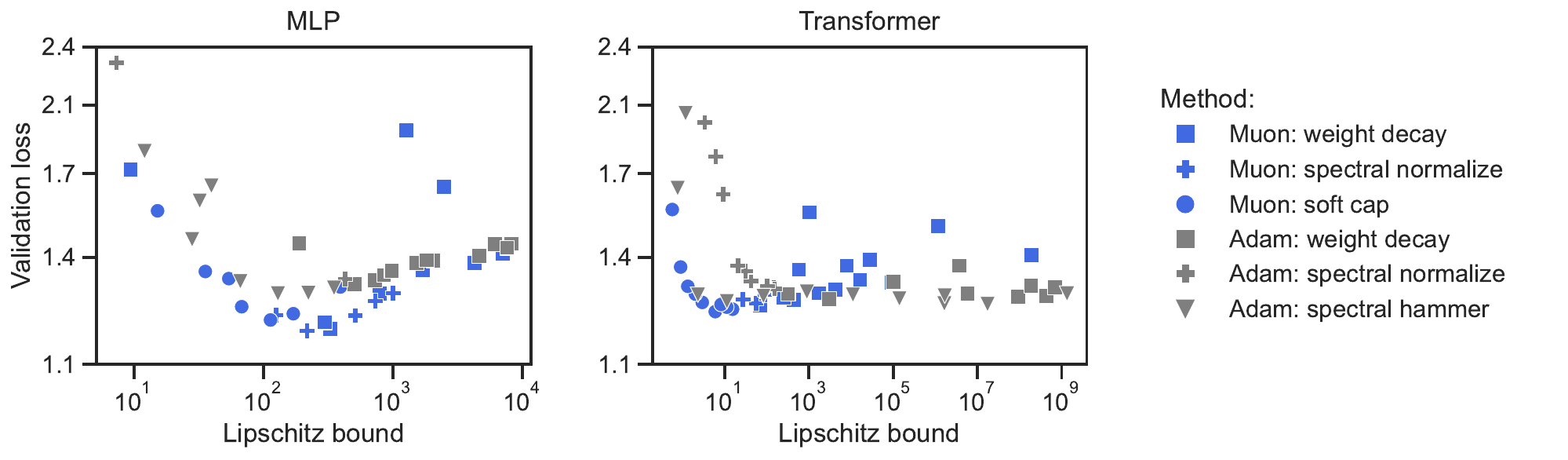}
  \caption{\textbf{Using Muon instead of AdamW improves the Lipschitz vs. performance tradeoff for standard weight regularization techniques.} We train 1600 MLPs on CIFAR-10 (left) and 800 transformers on Shakespeare text (right), varying the optimizer and weight constraint method. Weight decay and spectral normalization reach better loss with lower Lipschitz bound when using Muon. Two weight constraint methods that we design---spectral soft cap and spectral hammer---are also promising. See \cref{app:experimental-details} for experimental details.}
  \label{fig:pareto}
\end{figure}

We extend these ideas with two new methods:

\textbf{Spectral hammer} is similar to spectral weight decay, but sets the top singular value to $\sigma_\text{max}$ by mapping $W \mapsto W + (\sigma_\text{max} - \sigma_1)u_1v_1^\top$---so to speak, a hammer that strikes the nail that sticks out the most. Spectral hammer does not guarantee the spectral norm stays below $\sigma_\text{max} > 0$ because multiple singular vectors may increase per update. Spectral hammer is better suited to low stable rank weight updates as are common in Adam \citep{zhao2024adapproxadaptiveapproximationadam}. Muon's update is always high stable rank.

\textbf{Spectral capping} is co-designed for Muon's high stable rank update, smoothly approximating the map $\sigma \mapsto \min(\sigma_\text{max}, \sigma)$ for all singular values in parallel. Rather than rely on costly SVDs, it uses an odd polynomial approximation. The primary variant we experiment with is called \textit{spectral soft cap} because it applies a loose approximation $p_2(p_1(x))$, where $p_1(x) = x - \alpha x^3$ and $p_2(x) = x + \alpha x^3$ with strength parameter $\alpha \geq 0$, as discussed in \cref{app:spectral-softcap}. To incorporate weight decay as a special case, we may first apply $p_0(x) = (1 - \lambda \eta) x$. This composition is designed to decay a singular value very little when $\sigma \ll \sigma_\text{max}$, but when $\sigma = \sigma_\text{max}$ to decay it as strongly as necessary to counteract Muon's known update norm, strictly enforcing a weight norm bound. When the learning rate is scheduled, finding the minimal $\alpha \geq 0$ that bounds the weight norm avoids accumulating error.

\begin{theorem}[Spectral soft cap bounds spectral norm]
    \label{thm:spectral-soft-cap-bound}
    Given a desired maximum spectral norm $\sigma_\text{max} \geq 0$, learning rate $\eta \geq 0$, weight decay $\lambda \geq 0$, and weight matrix with bounded norm $\norm{W}_\ast \leq \sigma_\text{max}$, there is a minimal $\alpha \geq 0$ such that Muon's update step followed by spectral soft cap preserves $\norm{W}_\ast \leq \sigma_\text{max}$. Calculating $\alpha$ involves solving the roots of a quartic polynomial.
\end{theorem}

\cref{app:spectral-softcap} gives the proof. A second variant of spectral capping is \textit{spectral hard cap}, which uses the well-known matrix sign function to approximate the matrix function $\sigma \to \min(\sigma_\text{max}, \sigma)$ acting on the singular values, as discussed in \cref{app:spectral-hardcap}. Because this approximation is fixed, errors can compound later in training when the learning rate is scheduled to $0$.

Together, these methods cover a range of trade-offs between strict norm enforcement, preserving the spectrum, and computational efficiency. There may be other, better approaches, and we think exploring alternatives is an exciting direction. Note that spectral soft cap and spectral hard cap are designed to be compatible with Muon and therefore were not applied to tests with AdamW.

\subsection{AdamW and Muon: comparing weight constraint methods}

In \cref{fig:pareto}, we run a sweep to map the tradeoff frontier between validation loss and Lipschitz bound across our methods. The Lipschitz bounds are calculated with respect to the $\rms\to\rms$ operator norm (\cref{sec:methods_weight_constraints}). For a $\mathsf{ReLU}$ MLP, the Lipschitz bound is the product of the $\rms\to\rms$ norms of its weights. For a transformer, the Lipschitz bound is calculated as described in \cref{subsec:lipschitz-constant-of-transformer}. Muon consistently achieves both lower validation loss \textit{and} lower Lipschitz bounds than AdamW, a trend that holds for MLPs on CIFAR-10 and transformers on Shakespeare text. This result motivates our choice to adopt Muon for larger-scale experiments. Spectral normalization and spectral soft cap appear to make the most efficient use of a Lipschitz budget. Spectral hammer--which is designed for AdamW's low stable rank weight updates--shows competitive performance but does not enforce a Lipschitz bound, limiting its reliability for settings where constraint enforcement is critical. The sweep includes 2400 training runs, where for each method we display only the best validation loss per bin of Lipschitz bound.

\begin{figure}[t]
    \centering
    \includegraphics[width=1.0\linewidth]{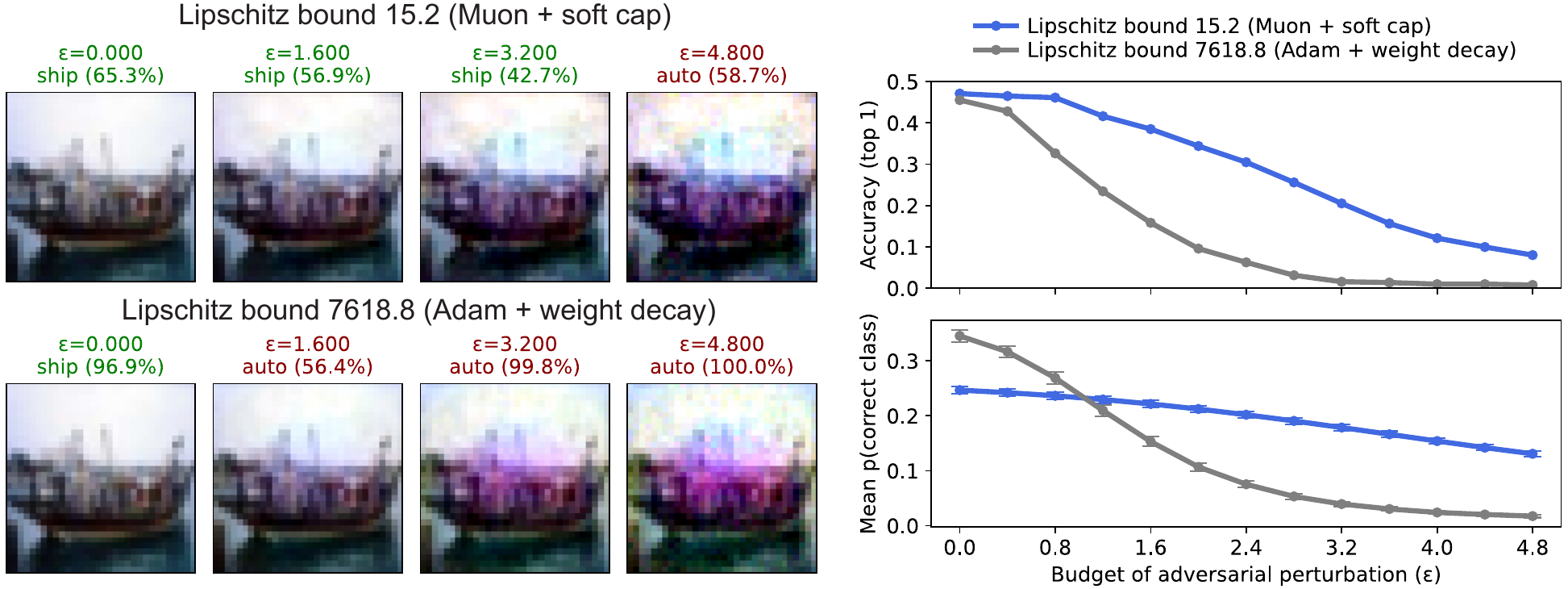}
    \caption{\textbf{Networks trained with Muon and spectral soft cap have lower Lipschitz bounds and are more adversarially robust.} 
    Lower Lipschitz bounds have been linked to greater adversarial robustness %
    \citep{cisse2017parseval,huang2021trainingcertifiablyrobustneural}. To assess this effect in our models, we train a CIFAR-10 MLP with a Lipschitz bound of 15.2 (Muon + spectral soft cap), which matches the 45\% clean accuracy of a baseline model (AdamW + weight decay) that has a a higher Lipschitz bound of 7618.8. 
    Left: Example adversarial attacks with different $\ell_2$ budget $\epsilon \geq 0$ for the perturbation. 
    Top right: We quantify adversarial robustness across 2000 test images by the top-1 accuracy as a function of $\epsilon$. 
    The Lipschitz-constrained network trained with Muon and spectral soft cap maintains a higher accuracy for larger values of $\epsilon$. 
    Bottom right: the mean probability of the correct class in the Lipschitz-constrained network starts lower than that of the baseline model, but degrades slowly under increasing $\epsilon$. By contrast, the baseline model peaks higher for $\epsilon=0$, but drops off sharply.
    }
    \label{fig:accuracy_vs_epsilon}
\end{figure}

\subsection{Adversarial robustness of Lipschitz networks}

Prior work \citep{cisse2017parseval,huang2021trainingcertifiablyrobustneural} suggests that a neural network's adversarial robustness is related to its Lipschitz constant, making Lipschitz control a potential path for developing models with high certified accuracy. We confirm this relationship holds for MLPs trained with Muon and spectral soft cap. \cref{fig:accuracy_vs_epsilon} compares the accuracy of two trained MLPs under various budgets of adversarial perturbation $\epsilon$. The pair depicted has matching baseline validation accuracy $\approx$ 45\% but different Lipschitz bounds: $15.2$ (Muon + spectral soft cap) vs. $7618.8$ (AdamW + weight decay).

While both models achieve similar accuracy without perturbation, the Lipschitz-constrained MLP (quantified in the RMS $\rightarrow$ RMS operator norm, \cref{sec:methods_weight_constraints}) trained with Muon exhibits a smoother dropoff in accuracy and confidence across the $2000$ CIFAR-10 test set images as the adversarial attack increases its $l_2$ perturbation. Larger values of $\epsilon$ are required to fool the Lipschitz-constrained network (Figure \ref{fig:accuracy_vs_epsilon}, Left).

\begin{figure}[t]
    \centering
    \includegraphics[width=1.0\linewidth]{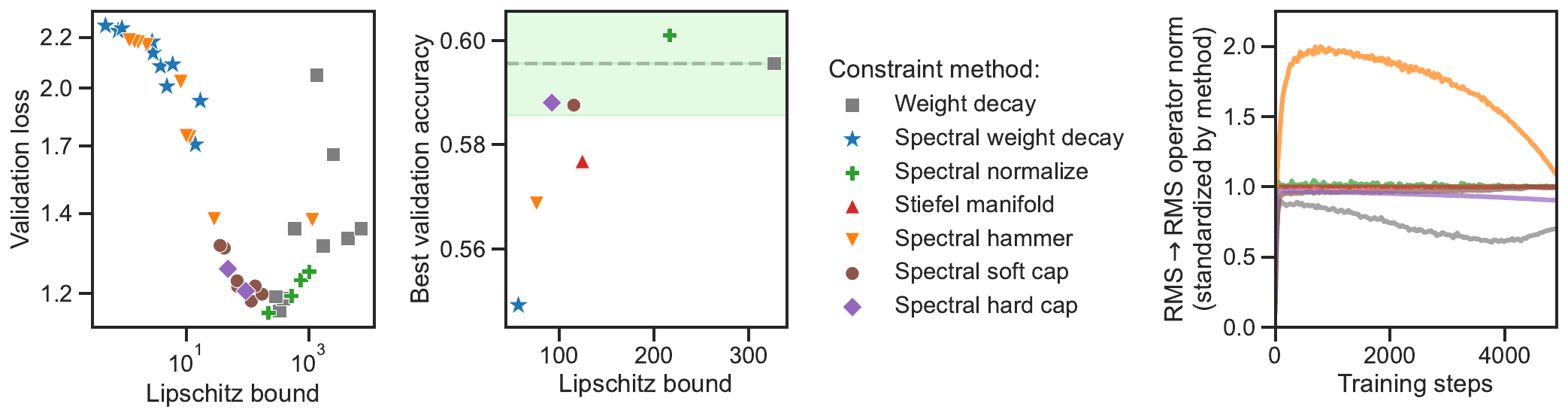}
    \caption{\textbf{Left: Weight constraint methods designed for Muon lie on the frontier of the Lipschitz vs. loss tradeoff.} Each point shows the lowest validation loss achieved at a given Lipschitz bound across all MLP runs on CIFAR-10. In the low loss regime, staying on the tradeoff frontier requires either spectral normalization or spectrally capping the singular values. \textbf{Middle: Spectral normalization and spectral capping match baseline with lower Lipschitz bound on CIFAR-10.} Each method comes within 1\% accuracy (shaded green region) and has lower Lipschitz bound. \textbf{Right: RMS $\boldsymbol{\to}$ RMS operator norm of hidden layers over training for the best networks.} The norm is standardized so that the weight constraints all target 1. Spectral normalization and Stiefel manifold projection strictly meet this target. Weight decay, spectral soft cap, and spectral hard cap stay below the target, while spectral hammer fails to remain constrained.}
    \label{fig:fullmuoncomp}
\end{figure}

\subsection{Comparing weight constraint methods within Muon} 

In \cref{fig:fullmuoncomp}, we use Muon alongside all seven weight constraint methods to identify which approaches best meet three goals: 1) enabling us to define a Lipschitz bound before training, 2) enforcing that bound throughout training, and 3) matching or exceeding the performance of standard weight decay. We tested each weight constraint method on a 3-layer MLP with hidden dimension 256, trained with Muon on CIFAR-10 classification. Full experimental details are in \cref{app:experimental-details}.

In \cref{fig:fullmuoncomp} (left panel), we see that spectral normalization, spectral soft cap, and spectral hard cap define the frontier of the Lipschitz vs. validation loss tradeoff. In \cref{fig:fullmuoncomp} (middle), we select the parameter settings with the highest validation accuracy for further analysis. Spectral normalization, spectral soft cap, and spectral hard cap are the only methods to reach within 1\% validation accuracy of the baseline (Muon with weight decay). Other methods reached within 4\% accuracy.

\cref{fig:fullmuoncomp} (right panel) visualizes how the norm of the hidden weight matrix evolves during training. All methods but spectral hammer retain the weight norm at or under its target. In contrast, spectral hammer exceeds its target but begins to converge near the end, indicating that it could be a promising technique under Muon, too, on longer training runs that also schedule learning rate to 0. However, within our 50-epoch window, it fails to reliably control the Lipschitz bound. Spectral weight decay does not enforce predefined Lipschitz bounds, so it is not plotted.

We select spectral soft cap, spectral hard cap, and spectral normalization for our transformer training experiments. While all weight constraint methods could potentially benefit from combining with standard weight decay, we want to isolate the effect of each one.

\section{Transformers with enforced weight constraints}
\label{sec:enforcing-for-transformers}

To develop a toolkit for training Lipschitz-constrained transformers, we begin with a discussion of how to handle residual connections and self-attention (\cref{subsec:architecture}). In \cref{subsec:lipschitz-constant-of-transformer}, we explain how to calculate a Lipschitz bound for a transformer. In \cref{subsec:shakespeare}, we find that spectral normalization and spectral soft cap perform well on 2M parameter transformers trained on Shakespeare text. In \cref{subsec:nanogpt}, we scale up to a transformer with 145M parameters trained on FineWeb10B internet text \citep{penedo2024finewebdatasetsdecantingweb}. To combat the undertuned baseline problem, we begin from the competitive benchmark of NanoGPT speedrunning \citep{modded_nanogpt_2024}.

\subsection{Breaking the multiplication barrier?}
\label{subsec:architecture}

A major triumph of \citet{large2024} is to create transformers with a depth-independent Lipschitz bounds. However, their approach relies on activations having unit $\rms$ norm. Because we will later relax this constraint, our transformers will typically \textit{not} have a depth-independent Lipschitz bound. Nonetheless, we use their two architectural suggestions.

\textbf{Reparameterizing residual connections.} The classic residual connection from \citet{he2015deepresiduallearningimage} defines the update $x + \mathsf{block}(x)$, but even a 1-Lipschitz block can exponentially increase the Lipschitz bound: $x + \mathsf{identity}(x)$ doubles at every layer. \citet{large2024} use a convex combination \begin{equation}
    \label{eq:breaking-barrier-convex-combination}
    \frac{N-1}{N} \cdot x + \frac{1}{N} \cdot\mathsf{block}(x)
\end{equation}to break this multiplication barrier,  where $N$ is the number of layers. The residual connection will be $1$-Lipschitz if the block is $1$-Lipschitz. See Proposition 4 of \citet{large2024}. Our experiments use this convex parameterization. However, the bound breaks down if activation norms exceed 1. We are unable to attain high performance without relaxing the 1-Lipschitz constraint. Therefore we do not fully break the multiplication barrier: deeper networks can accrue astronomical Lipschitz bounds.

\textbf{Attention with $\boldsymbol{1/d}$ scaling.} The original multihead attention of \citet{vaswani2017attention} has no global Lipschitz bound \citep{kim2021lipschitzconstantselfattention}. In a footnote, \citet{vaswani2017attention} indicate that they chose $1/\sqrt{d}$ scaling because two random vectors $u, v$ with mean $0$ and variance $1$ will have a dot product $u \cdot v$ with mean $0$ and variance the dimension $d$. But key and query vectors may align more than at random. Perfect alignment would suggest $1/d$ scaling. \citet{large2024} prove that $1/d$ scaling in the softmax makes functional attention, \begin{equation}
    \label{eq:breaking-barrier-softmax-1-d}
    \softmax\left(\frac{QK^\top}{d}\right)V,
\end{equation} 1-Lipschitz if the input norms are 1. The Lipschitz bound is with respect to the $\norm{\cdot}_{\infty\rms}$ norm, the max $\rms$ norm of any token. We then multiply by $\tfrac13$ to make attention composed with the 3-sensitive input tuple $(q, k, v)$ unit sensitivity. See Proposition 7 of \citet{large2024}. Finally, while \citet{large2024} retains layer normalization, we remove it so that every operation is Lipschitz continuous.

\subsection{Calculating the Lipschitz bound of a transformer}
\label{subsec:lipschitz-constant-of-transformer}

To test whether transformers with small Lipschitz bound can perform well, we would like to know what weight norm to enforce to end up with a desired Lipschitz bound. This section sketches an algorithm for bounding the Lipschitz constant of a transformer given its weight norms; the full algorithm is in \cref{app:lipschitz-attention}. Our bound tightens Theorem 2 from LipsFormer \citep{qi2023lipsformer} by accounting for each weight norm individually, rather than working with the max weight norm across residual blocks. \citet{fazlyab2023efficientaccurateestimationlipschitz} suggest ways to tighten a global bound like ours, which we leave to future work. To bound self-attention, we extend Proposition 7 from the modular norm paper \citep{large2024} to when the input is no longer unit norm, an essential concession to reach different regions of the Lipschitz vs. performance tradeoff.

Our Lipschitz bounds are with respect to the max $\rms$ norm over token positions, denoted $\norm{\cdot}_{\infty\rms}$.

\textbf{Step 1: Bound activation norms.} Our Lipschitz bound for rescaled dot-product attention relies on the activation norms remaining bounded. Therefore, the algorithm begins by computing a per-layer bound for the maximum $\norm{\cdot}_{\infty\rms}$ norm of activations. This bound comes from combining residual connections with per-block maximum activation increases, found for an MLP by multiplying its two weight norms and for attention by multiplying its $W_V$ and $W_O$ weight norms. Our MLPs can slightly decay activation norm because we scale $\mathsf{GeLU}$ down by its maximum derivative, $\mathsf{GeLU}/1.1289$.

\textbf{Step 2: Bound the Lipschitz constant.} Suppose the Lipschitz constant prior to reaching a particular layer is $L$. The Lipschitz constant after a residual connection is at most $(1 - \alpha) L + \alpha \cdot L \cdot L_\mathsf{block}$. To compute a block's Lipschitz bound $L_\mathsf{block}$, for MLPs we calculate $\norm{W_\mathsf{in}}_{\rms\to\rms} \cdot \norm{W_\mathsf{out}}_{\rms\to\rms} / 1.1289$ due our rescaled $\mathsf{GeLU}$, and for attention we use \cref{thm:function-attention-is-lipschitz}.

\subsection{Shakespeare Transformer}
\label{subsec:shakespeare}

Before scaling to NanoGPT, we explore the Lipschitz vs. performance tradeoff for transformers at a smaller scale. To narrow our aim, we want a transformer that has small Lipschitz bound \textit{at every layer}. \citet{bethune2022pay} comments that any $L$-Lipschitz classifier can be made $1$-Lipschitz by dividing the logits by $L$, but we do not downscale our final logits, although scaling temperature during training could have beneficial effects \citep{agarwala2020temperaturechecktheorypractice}. All experimental details are available in \cref{app:experimental-details}.

\textbf{We can train a competitive \textless 2-Lipschitz Shakespeare transformer}. Our \textless 2-Lipschitz transformer reaches validation loss $1.29 < 1.47$ from the baseline in \citet{karpathy_nanoGPT}, although the baseline may not be tuned carefully. Our model has dimension 256, depth 3, and was trained for 2000 steps with Muon; the baseline has dimension 384, depth 6, and was trained for 5000 steps with AdamW. Our transformer does not use layer normalization. To achieve this performance requires relaxing the maximum weight norm $\sigma_\text{max}$ to around 2. The best validation loss from our sweep was 1.20 with a \textless 6.02-Lipschitz transformer. While gains may be attributed to hyperparameters or the choice of optimizer, we nonetheless end up with one construction of our aim: a performant transformer with a small, enforced Lipschitz bound.

\begin{table}[tbp]
    \centering
    \resizebox{\textwidth}{!}{
    \renewcommand{\arraystretch}{1.1}
    \begin{tabular}{lcccccc}
        \toprule
        \makecell[l]{\textbf{Transformer}\\\textbf{Architecture}} &
        \makecell[c]{\textbf{Lipschitz}\\\textbf{Bound}} &
        \makecell[c]{\textbf{Tokens}\\\textbf{Used}} &
        \makecell[c]{\textbf{Weight}\\\textbf{Constraint}} &
        \makecell[c]{\textbf{Validation}\\\textbf{Accuracy} ($\uparrow$)} &
        \makecell[c]{\textbf{Validation}\\\textbf{Loss} ($\downarrow$)} &
        \makecell[c]{\textbf{Activation}\\\textbf{Max Entry}} \\ %
        \midrule
        Baseline (Modula) & $\infty$ & 0.7B & none & \textbf{0.374} & \textbf{3.491} & 1464 \\
        LipsFormer & $10^{130}$ & 0.7B & none & 0.301 & 4.130 & 61 \\ %
        Ours ($\sigma_\text{max}=1$) & $10$ & 0.7B & spectral normalize & 0.212 & 5.047 & 14 \\
        Ours ($\sigma_\text{max}=8$) & $10^{110}$ & 0.7B & spectral soft cap & 0.365 & 3.567 & 292 \\
        Ours ($\sigma_\text{max}=8$) & $10^{105}$ & 0.7B & spectral hard cap & 0.348 & 3.810 & 71 \\
        Ours ($\sigma_\text{max}=8$) & $10^{134}$ & 0.7B & spectral normalize & 0.362 & 3.582 & 103 \\
        Ours ($\sigma_\text{max}=\infty$) & $\infty$ & 0.7B & none & nan & nan & nan \\
        \midrule
        Baseline (Speedrun) & $\infty$ & 0.7B & none & 0.394 & \textbf{3.280} & 148{,}480 \\
        Baseline (Karpathy) & $\infty$ & 8.9B & none & - & \textbf{3.280} & - \\
        Ours ($\sigma_\text{max}=16$) & $10^{264}$ & 2.8B & spectral normalize & \textbf{0.395} & \textbf{3.280} & 160 \\
        \bottomrule
    \end{tabular}
    }
    \vspace{0.6em}
    \caption{\textbf{Transformers with enforced Lipschitz constraints can match performance on NanoGPT.} The NanoGPT speedrun is a competitively tuned benchmark building on Karpathy's original replication of GPT-2 \citep{karpathy_nanoGPT,modded_nanogpt_2024}. With the speedrun baseline as a starting point, we substitute Lipschitz transformer components and constrain weight norms to not exceed a given $\sigma_\text{max} \geq 0$. Unlike LipsFormer \citep{qi2023lipsformer}, our weight constraints enforce a Lipschitz bound chosen prior to training. To demonstrate, we train a 10-Lipschitz transformer to $21.2\%$ accuracy; the model uses no layer norm, QK norm, or logit tanh, yet trains stably. The same model with no weight constraints diverges. However, matching the baseline accuracy increases the Lipschitz bound to $10^{264}$, computed as in \cref{subsec:lipschitz-constant-of-transformer}. Our Lipschitz bounds may be loose, as suggested by the small maximum activation that we observe across a batch of 393K tokens. Final loss variance is $0.0008$.}
    \label{tab:nanogptspeedrun}
\end{table}

\subsection{Scaling to NanoGPT}
\label{subsec:nanogpt}

We validate our method by training a transformer with 145M parameters on the NanoGPT speedrun benchmark \citep{modded_nanogpt_2024}, which is built on top of \citep{karpathy_nanoGPT}'s reproduction of GPT-2. The baseline is competitively optimized to reach validation loss 3.28 in the shortest wallclock time. The latest record as of February 1, 2025 requires only 0.7B tokens, or 3 minutes of training on an 8xH100, and achieves a validation accuracy of 39.4\%. We implement our methods on top of the speedrun while keeping all other training methods fixed. We consider two baselines: 1) the original speedrun, and 2) the speedrun already modified to incorporate residual reparameterization and $\tfrac{1}{d}$ attention scaling from \cref{subsec:architecture}, but with stability measures like layer norm removed. We find this second baseline diverges during training, indicating weight constraints are necessary to proceed without traditional stability techniques. We report the validation loss and the validation accuracy as primary comparison metrics.

To implement our method, we first remove speedrun-specific optimizations such as skip connections and learnable scale parameters. We remove the layernorms used for pre-normalization, the logit tanh softcap, and the QK norms in the attention layers. We also replace the $\mathsf{ReLU}^2$ activations with $\mathsf{GeLU}/1.1289$---making the activation function Lipschitz continuous. We reparametrize the residual connections and attention layer as in \cref{eq:breaking-barrier-convex-combination,eq:breaking-barrier-softmax-1-d}. At initialization, we project the linear weights to be semi-orthogonal and normalize the embeddings to have $\rms$ norm 1. Finally, we explicitly extend beyond the closest related work, LipsFormer \citep{qi2023lipsformer}, by enforcing weight norm constraints throughout training: we cap the $\rms$ norm of embeddings to 1 and apply one of the weight constraint methods from \cref{sec:methods_weight_constraints} to all other weights after every training step. Matrix multiplications in the linear layers in the MLPs and attention are done in fp8 precision.

\cref{tab:nanogptspeedrun} summarizes our 145M parameter scale transformer results. In contrast to LipsFormer, our method guarantees a Lipschitz upper bound specified prior to training. An important lever is the maximum $\rms\to\rms$ norm $\sigma_\text{max} \geq 0$ we allow for the linear layers. Smaller $\sigma_\text{max}$ correspond to smaller Lipschitz bounds, but may lower performance and vice versa. Two other variables that affect the Lipschitz bound are the attention logit scale and the final logit scale. Using spectral normalization with $\sigma_\text{max}=1$ and a final logit scale of 8 results in a 10-Lipschitz transformer that achieves validation loss 5.047 and accuracy 21.2\%. No activation in this model exceeds $\rms$ norm 1, which could enhance stability during training. Setting $\sigma_\text{max}=16$ enables parity with speedrun performance, but increases the Lipschitz bound to $10^{264}$.
This setting trains with $4\times$ as many tokens as the current speedrun record, but $3.2\times$ fewer than Karpathy's baseline.

\section{Discussion}
\label{sec:discussion}

Despite high Lipschitz upper bounds, our transformers on NanoGPT exhibit low maximum activation entries (50-110) compared to the baseline (148K). Perhaps as a result, our models train stably without standard measures including layer norm, QK norm, or tanh logit softcapping. In future work, we are interested to test whether these low maximum activations hold potential for low-precision training and inference. We also wonder whether training would remain stable at larger scales.

For MLPs and small transformers, we find that using Muon improves the Lipschitz vs. performance tradeoff. Out of weight constraint methods we test, \textit{spectral normalization}, \textit{spectral soft cap}, and \textit{spectral hard cap} compare favorably to standard weight decay. Perhaps surprisingly, on both CIFAR-10 and Shakespeare data, we achieve our best loss with Lipschitz-enforced models, potentially representing a training speed benefit.

Our work has several limitations. We did not find a principled way to select weight norm, final logit scale, and attention logit scale hyperparameters, instead relying on sweeps. Our Lipschitz bound also increases rapidly as depth increases, unless we constrain weights to unit norm. A different architecture, or insight beyond a global Lipschitz bound, could make progress on this problem.

In conclusion, this paper develops a method for training transformers with an enforced Lipschitz bound throughout training, extending earlier efforts focused on different architectures or only constraining at initialization. Lipschitz-certified transformers may be of particular interest for domains such as privacy, control, adversarial robustness, low-precision training, and loss-spike-free large scale pretraining. Although training speed benefits fade in our NanoGPT speedrun experiments, we wonder whether at this scale Lipschitz-enforced training can be made faster than standard training.

\section*{Acknowledgements}

Thank you to Lambda Labs, Rami Seid, and Lucid Simulations for compute support. This work was supported by a Packard Fellowship to P.I. and by ONR MURI grant N00014-22-1-2740.

\bibliographystyle{plainnat}
\bibliography{refs}

\begin{thebibliography}{57}
\providecommand{\natexlab}[1]{#1}
\providecommand{\url}[1]{\texttt{#1}}
\expandafter\ifx\csname urlstyle\endcsname\relax
  \providecommand{\doi}[1]{doi: #1}\else
  \providecommand{\doi}{doi: \begingroup \urlstyle{rm}\Url}\fi

\bibitem[Agarwala et~al.(2023)Agarwala, Schoenholz, Pennington, and Dauphin]{agarwala2020temperaturechecktheorypractice}
Atish Agarwala, Samuel~Stern Schoenholz, Jeffrey Pennington, and Yann Dauphin.
\newblock Temperature check: Theory and practice for training models with softmax-cross-entropy losses.
\newblock \emph{Transactions on Machine Learning Research}, 2023.

\bibitem[Anil et~al.(2019)Anil, Lucas, and Grosse]{anil2019sorting}
Cem Anil, James Lucas, and Roger~B. Grosse.
\newblock Sorting out {L}ipschitz function approximation.
\newblock In \emph{International Conference on Machine Learning}, 2019.

\bibitem[Arjovsky et~al.(2016)Arjovsky, Shah, and Bengio]{arjovsky2016unitaryevolutionrecurrentneural}
Martin Arjovsky, Amar Shah, and Yoshua Bengio.
\newblock Unitary evolution recurrent neural networks.
\newblock In \emph{International Conference on Machine Learning}, 2016.

\bibitem[Ba et~al.(2016)Ba, Kiros, and Hinton]{ba2016layernormalization}
Jimmy~Lei Ba, Jamie~Ryan Kiros, and Geoffrey~E. Hinton.
\newblock Layer normalization.
\newblock \emph{arXiv:1607.06450}, 2016.

\bibitem[Bartlett et~al.(2017)Bartlett, Foster, and Telgarsky]{bartlett2017spectrallynormalizedmarginboundsneural}
Peter Bartlett, Dylan~J. Foster, and Matus Telgarsky.
\newblock Spectrally-normalized margin bounds for neural networks.
\newblock In \emph{Neural Information Processing Systems}, 2017.

\bibitem[Bernstein(2025)]{modula-docs}
Jeremy Bernstein.
\newblock The {M}odula docs, 2025.
\newblock URL \url{https://docs.modula.systems/}.
\newblock MIT License.

\bibitem[Bernstein and Newhouse(2025)]{bernstein2024modulardualitydeeplearning}
Jeremy Bernstein and Laker Newhouse.
\newblock Modular duality in deep learning.
\newblock In \emph{International Conference on Machine Learning}, 2025.

\bibitem[B{\'e}thune(2024)]{bethune2024deep}
Louis B{\'e}thune.
\newblock \emph{Deep Learning with {Lipschitz} Constraints}.
\newblock PhD thesis, Universit{\'e} de Toulouse, 2024.

\bibitem[B{\'e}thune et~al.(2022)B{\'e}thune, Boissin, Serrurier, Mamalet, Friedrich, and Sanz]{bethune2022pay}
Louis B{\'e}thune, Thibaut Boissin, Mathieu Serrurier, Franck Mamalet, Corentin Friedrich, and Alberto~Gonzalez Sanz.
\newblock Pay attention to your loss: Understanding misconceptions about {L}ipschitz neural networks.
\newblock In \emph{Neural Information Processing Systems}, 2022.

\bibitem[B{\'e}thune et~al.(2024)B{\'e}thune, Massena, Boissin, Prudent, Friedrich, Mamalet, Bellet, Serrurier, and Vigouroux]{bethune2024dpsgdclippinglipschitzneural}
Louis B{\'e}thune, Thomas Massena, Thibaut Boissin, Yannick Prudent, Corentin Friedrich, Franck Mamalet, Aurelien Bellet, Mathieu Serrurier, and David Vigouroux.
\newblock {DP-SGD} without clipping: The {L}ipschitz neural network way.
\newblock In \emph{International Conference on Learning Representations}, 2024.

\bibitem[Bradbury et~al.(2018)Bradbury, Frostig, Hawkins, Johnson, Leary, Maclaurin, Necula, Paszke, Vander{P}las, Wanderman-{M}ilne, and Zhang]{jax2018github}
James Bradbury, Roy Frostig, Peter Hawkins, Matthew~James Johnson, Chris Leary, Dougal Maclaurin, George Necula, Adam Paszke, Jake Vander{P}las, Skye Wanderman-{M}ilne, and Qiao Zhang.
\newblock {JAX}: Composable transformations of {P}ython+{N}um{P}y programs, 2018.
\newblock URL \url{http://github.com/jax-ml/jax}.

\bibitem[Brown-Cohen et~al.(2024)Brown-Cohen, Irving, and Piliouras]{brown-cohen2024scalabale}
Jonah Brown-Cohen, Geoffrey Irving, and Georgios Piliouras.
\newblock Scalable {AI} safety via doubly-efficient debate.
\newblock In \emph{International Conference on Machine Learning}, 2024.

\bibitem[Cesista(2025)]{cesista2025spectralclipping}
Franz~Louis Cesista.
\newblock Fast, numerically stable, and auto-differentiable {Spectral Clipping} via {Newton-Schulz} iteration, June 2025.
\newblock URL \url{http://leloykun.github.io/ponder/spectral-clipping/}.

\bibitem[Cesista et~al.(2025)Cesista, You, and Jordan]{cesista2025muonoptcoeffs}
Franz~Louis Cesista, Jiacheng You, and Keller Jordan.
\newblock Squeezing 1--2\% efficiency gains out of {M}uon by optimizing the {N}ewton-{S}chulz coefficients, 2025.
\newblock URL \url{http://leloykun.github.io/ponder/muon-opt-coeffs/}.

\bibitem[Cisse et~al.(2017)Cisse, Bojanowski, Grave, Dauphin, and Usunier]{cisse2017parseval}
Moustapha Cisse, Piotr Bojanowski, Edouard Grave, Yann Dauphin, and Nicolas Usunier.
\newblock Parseval networks: Improving robustness to adversarial examples.
\newblock In \emph{International Conference on Machine Learning}, 2017.

\bibitem[Dehghani et~al.(2023)Dehghani, Djolonga, Mustafa, Padlewski, Heek, Gilmer, Steiner, Caron, Geirhos, Alabdulmohsin, Jenatton, Beyer, Tschannen, Arnab, Wang, Riquelme, Minderer, Puigcerver, Evci, Kumar, van Steenkiste, Elsayed, Mahendran, Yu, Oliver, Huot, Bastings, Collier, Gritsenko, Birodkar, Vasconcelos, Tay, Mensink, Kolesnikov, Pavetić, Tran, Kipf, Lučić, Zhai, Keysers, Harmsen, and Houlsby]{dehghani2023scalingvisiontransformers22}
Mostafa Dehghani, Josip Djolonga, Basil Mustafa, Piotr Padlewski, Jonathan Heek, Justin Gilmer, Andreas Steiner, Mathilde Caron, Robert Geirhos, Ibrahim Alabdulmohsin, Rodolphe Jenatton, Lucas Beyer, Michael Tschannen, Anurag Arnab, Xiao Wang, Carlos Riquelme, Matthias Minderer, Joan Puigcerver, Utku Evci, Manoj Kumar, Sjoerd van Steenkiste, Gamaleldin~F. Elsayed, Aravindh Mahendran, Fisher Yu, Avital Oliver, Fantine Huot, Jasmijn Bastings, Mark~Patrick Collier, Alexey Gritsenko, Vighnesh Birodkar, Cristina Vasconcelos, Yi~Tay, Thomas Mensink, Alexander Kolesnikov, Filip Pavetić, Dustin Tran, Thomas Kipf, Mario Lučić, Xiaohua Zhai, Daniel Keysers, Jeremiah Harmsen, and Neil Houlsby.
\newblock Scaling vision transformers to 22 billion parameters.
\newblock In \emph{International Conference on Machine Learning}, 2023.

\bibitem[Dherin et~al.(2022)Dherin, Munn, Rosca, and Barrett]{dherin2022why}
Benoit Dherin, Michael Munn, Mihaela Rosca, and David~GT Barrett.
\newblock Why neural networks find simple solutions: The many regularizers of geometric complexity.
\newblock In \emph{Neural Information Processing Systems}, 2022.

\bibitem[Elthakeb et~al.(2020)Elthakeb, Pilligundla, Mireshghallah, Cloninger, and Esmaeilzadeh]{pmlr-v119-elthakeb20a}
Ahmed~Taha Elthakeb, Prannoy Pilligundla, Fatemeh Mireshghallah, Alexander Cloninger, and Hadi Esmaeilzadeh.
\newblock Divide and conquer: Leveraging intermediate feature representations for quantized training of neural networks.
\newblock In \emph{International Conference on Machine Learning}, 2020.

\bibitem[Fazlyab et~al.(2019)Fazlyab, Robey, Hassani, Morari, and Pappas]{fazlyab2023efficientaccurateestimationlipschitz}
Mahyar Fazlyab, Alexander Robey, Hamed Hassani, Manfred Morari, and George~J. Pappas.
\newblock Efficient and accurate estimation of {L}ipschitz constants for deep neural networks.
\newblock In \emph{Neural Information Processing Systems}, 2019.

\bibitem[Flynn(2017)]{flynn2017duality}
Thomas Flynn.
\newblock The duality structure gradient descent algorithm: Analysis and applications to neural networks.
\newblock \emph{arXiv:1708.00523}, 2017.

\bibitem[Gogianu et~al.(2021)Gogianu, Berariu, Rosca, Clopath, Buşoniu, and Pascanu]{Gogianu2021SpectralNF}
Florin Gogianu, Tudor Berariu, Mihaela Rosca, Claudia Clopath, Lucian Buşoniu, and Razvan Pascanu.
\newblock Spectral normalisation for deep reinforcement learning: An optimisation perspective.
\newblock In \emph{International Conference on Machine Learning}, 2021.

\bibitem[Gouk et~al.(2021)Gouk, Frank, Pfahringer, and Cree]{gouk2020regularisationneuralnetworksenforcing}
Henry Gouk, Eibe Frank, Bernhard Pfahringer, and Michael~J. Cree.
\newblock Regularisation of neural networks by enforcing {L}ipschitz continuity.
\newblock \emph{Machine Learning}, 2021.

\bibitem[He et~al.(2016)He, Zhang, Ren, and Sun]{he2015deepresiduallearningimage}
Kaiming He, Xiangyu Zhang, Shaoqing Ren, and Jian Sun.
\newblock Deep residual learning for image recognition.
\newblock In \emph{Computer Vision and Pattern Recognition}, 2016.

\bibitem[Henry et~al.(2020)Henry, Dachapally, Pawar, and Chen]{henry2020querykeynormalizationtransformers}
Alex Henry, Prudhvi~Raj Dachapally, Shubham Pawar, and Yuxuan Chen.
\newblock Query-key normalization for transformers.
\newblock In \emph{Empirical Methods in Natural Language Processing}, 2020.

\bibitem[Huang et~al.(2021)Huang, Zhang, Shi, Kolter, and Anandkumar]{huang2021trainingcertifiablyrobustneural}
Yujia Huang, Huan Zhang, Yuanyuan Shi, J.~Zico Kolter, and Anima Anandkumar.
\newblock Training certifiably robust neural networks with efficient local {L}ipschitz bounds.
\newblock In \emph{Neural Information Processing Systems}, 2021.

\bibitem[Jianlin(2024)]{kexuefm-spectral-weight-decay}
Su~Jianlin.
\newblock Thoughts from spectral norm gradient to new weight decay, Dec 2024.
\newblock URL \url{https://kexue.fm/archives/10648}.

\bibitem[Jordan et~al.(2024{\natexlab{a}})Jordan, Bernstein, Rappazzo, @fernbear.bsky.social, Vlado, Jiacheng, Cesista, Koszarsky, and @Grad62304977]{modded_nanogpt_2024}
Keller Jordan, Jeremy Bernstein, Brendan Rappazzo, @fernbear.bsky.social, Boza Vlado, You Jiacheng, Franz Cesista, Braden Koszarsky, and @Grad62304977.
\newblock Modded-{nanoGPT}: Speedrunning the {nanoGPT} baseline, 2024{\natexlab{a}}.
\newblock URL \url{https://github.com/KellerJordan/modded-nanogpt}.
\newblock MIT License.

\bibitem[Jordan et~al.(2024{\natexlab{b}})Jordan, Jin, Boza, Jiacheng, Cesista, Newhouse, and Bernstein]{jordan2024muon}
Keller Jordan, Yuchen Jin, Vlado Boza, You Jiacheng, Franz Cesista, Laker Newhouse, and Jeremy Bernstein.
\newblock Muon: An optimizer for hidden layers in neural networks, 2024{\natexlab{b}}.
\newblock URL \url{https://kellerjordan.github.io/posts/muon/}.

\bibitem[Karpathy(2022)]{karpathy_nanoGPT}
Andrej Karpathy.
\newblock {nanoGPT}.
\newblock \url{https://github.com/karpathy/nanoGPT}, 2022.
\newblock MIT License.

\bibitem[Katz et~al.(2017)Katz, Barrett, Dill, Julian, and Kochenderfer]{katz}
Guy Katz, Clark Barrett, David~L. Dill, Kyle Julian, and Mykel~J. Kochenderfer.
\newblock Reluplex: An efficient {SMT} solver for verifying deep neural networks.
\newblock In \emph{International Conference on Computer Aided Verification}, 2017.

\bibitem[Kim et~al.(2021)Kim, Papamakarios, and Mnih]{kim2021lipschitzconstantselfattention}
Hyunjik Kim, George Papamakarios, and Andriy Mnih.
\newblock The {L}ipschitz constant of self-attention.
\newblock In \emph{International Conference on Machine Learning}, 2021.

\bibitem[Krizhevsky(2009)]{krizhevsky2009cifar}
Alex Krizhevsky.
\newblock Learning multiple layers of features from tiny images.
\newblock Technical report, University of Toronto, 2009.

\bibitem[Krogh and Hertz(1991)]{krogh1991simple}
Anders Krogh and John~A. Hertz.
\newblock A simple weight decay can improve generalization.
\newblock In \emph{Neural Information Processing Systems}, 1991.

\bibitem[Large et~al.(2024)Large, Liu, Huh, Bahng, Isola, and Bernstein]{large2024}
Tim Large, Yang Liu, Minyoung Huh, Hyojin Bahng, Phillip Isola, and Jeremy Bernstein.
\newblock Scalable optimization in the modular norm.
\newblock In \emph{Neural Information Processing Systems}, 2024.

\bibitem[Loshchilov and Hutter(2019)]{loshchilov2019decoupledweightdecayregularization}
Ilya Loshchilov and Frank Hutter.
\newblock Decoupled weight decay regularization.
\newblock In \emph{International Conference on Learning Representations}, 2019.

\bibitem[Miyato et~al.(2018)Miyato, Kataoka, Koyama, and Yoshida]{miyato2018spectralnormalizationgenerativeadversarial}
Takeru Miyato, Toshiki Kataoka, Masanori Koyama, and Yuichi Yoshida.
\newblock Spectral normalization for generative adversarial networks.
\newblock In \emph{International Conference on Learning Representations}, 2018.

\bibitem[{Moonshot AI}(2025)]{kimik2_2025}
{Moonshot AI}.
\newblock Kimi k2: Open agentic intelligence, July 2025.
\newblock URL \url{https://moonshotai.github.io/Kimi-K2/}.
\newblock Technical report on Kimi K2, a 1T parameter Mixture-of-Experts model with 32B activated parameters.

\bibitem[Neyshabur et~al.(2018)Neyshabur, Bhojanapalli, and Srebro]{neyshabur2018a}
Behnam Neyshabur, Srinadh Bhojanapalli, and Nathan Srebro.
\newblock A {PAC}-bayesian approach to spectrally-normalized margin bounds for neural networks.
\newblock In \emph{International Conference on Learning Representations}, 2018.

\bibitem[O’Connell et~al.(2022)O’Connell, Shi, Shi, Azizzadenesheli, Anandkumar, Yue, and Chung]{neural-fly}
Michael O’Connell, Guanya Shi, Xichen Shi, Kamyar Azizzadenesheli, Anima Anandkumar, Yisong Yue, and Soon-Jo Chung.
\newblock Neural-fly enables rapid learning for agile flight in strong winds.
\newblock \emph{Science Robotics}, 2022.

\bibitem[Penedo et~al.(2024)Penedo, Kydlíček, allal, Lozhkov, Mitchell, Raffel, Werra, and Wolf]{penedo2024finewebdatasetsdecantingweb}
Guilherme Penedo, Hynek Kydlíček, Loubna~Ben allal, Anton Lozhkov, Margaret Mitchell, Colin Raffel, Leandro~Von Werra, and Thomas Wolf.
\newblock The {FineWeb} datasets: Decanting the web for the finest text data at scale.
\newblock In \emph{Neural Information Processing Systems: Datasets and Benchmarks Track}, 2024.
\newblock ODC-By 1.0 License.

\bibitem[Pethick et~al.(2025)Pethick, Xie, Antonakopoulos, Zhu, Silveti-Falls, and Cevher]{pethick2025trainingdeeplearningmodels}
Thomas Pethick, Wanyun Xie, Kimon Antonakopoulos, Zhenyu Zhu, Antonio Silveti-Falls, and Volkan Cevher.
\newblock Training deep learning models with norm-constrained {LMO}s.
\newblock \emph{arXiv:2502.07529}, 2025.

\bibitem[Qi et~al.(2023)Qi, Wang, Chen, Shi, and Zhang]{qi2023lipsformer}
Xianbiao Qi, Jianan Wang, Yihao Chen, Yukai Shi, and Lei Zhang.
\newblock Lips{F}ormer: Introducing {L}ipschitz continuity to vision transformers.
\newblock In \emph{International Conference on Learning Representations}, 2023.

\bibitem[Rosca et~al.(2020)Rosca, Weber, Gretton, and Mohamed]{pmlr-v137-rosca20a}
Mihaela Rosca, Theophane Weber, Arthur Gretton, and Shakir Mohamed.
\newblock A case for new neural network smoothness constraints.
\newblock In \emph{NeurIPS Workshop on "I Can't Believe It's Not Better!"}, 2020.

\bibitem[Su(2025)]{kexuefm-10795}
Jianlin Su.
\newblock Higher-order mup: A more concise but more intelligent spectral condition scaling, Mar 2025.
\newblock URL \url{https://kexue.fm/archives/10795}.

\bibitem[Szegedy et~al.(2014)Szegedy, Zaremba, Sutskever, Bruna, Erhan, Goodfellow, and Fergus]{szegedy2013intriguing}
Christian Szegedy, Wojciech Zaremba, Ilya Sutskever, Joan Bruna, Dumitru Erhan, Ian Goodfellow, and Rob Fergus.
\newblock Intriguing properties of neural networks.
\newblock In \emph{International Conference on Learning Representations}, 2014.

\bibitem[Tsuzuku et~al.(2018)Tsuzuku, Sato, and Sugiyama]{tsuzuku2018lipschitzmargintrainingscalablecertification}
Yusuke Tsuzuku, Issei Sato, and Masashi Sugiyama.
\newblock Lipschitz-margin training: Scalable certification of perturbation invariance for deep neural networks.
\newblock In \emph{Neural Information Processing Systems}, 2018.

\bibitem[Vaswani et~al.(2017)Vaswani, Shazeer, Parmar, Uszkoreit, Jones, Gomez, Kaiser, and Polosukhin]{vaswani2017attention}
Ashish Vaswani, Noam Shazeer, Niki Parmar, Jakob Uszkoreit, Llion Jones, Aidan~N Gomez, {\L}ukasz Kaiser, and Illia Polosukhin.
\newblock Attention is all you need.
\newblock In \emph{Neural Information Processing Systems}, 2017.

\bibitem[Virmaux and Scaman(2018)]{NEURIPS2018_d54e99a6}
Aladin Virmaux and Kevin Scaman.
\newblock Lipschitz regularity of deep neural networks: Analysis and efficient estimation.
\newblock In \emph{Neural Information Processing Systems}, 2018.

\bibitem[Wang et~al.(2020)Wang, Weng, and Daniel]{Wang2019VerificationON}
Yuh-Shyang Wang, Tsui-Wei Weng, and Luca Daniel.
\newblock Verification of neural network control policy under persistent adversarial perturbation.
\newblock In \emph{International Conference on Machine Learning}, 2020.

\bibitem[Weng et~al.(2018)Weng, Zhang, Chen, Song, Hsieh, Daniel, Boning, and Dhillon]{pmlr-v80-weng18a}
Lily Weng, Huan Zhang, Hongge Chen, Zhao Song, Cho-Jui Hsieh, Luca Daniel, Duane Boning, and Inderjit Dhillon.
\newblock Towards fast computation of certified robustness for {R}e{LU} networks.
\newblock In \emph{International Conference on Machine Learning}, 2018.

\bibitem[Weng et~al.(2019)Weng, Chen, Nguyen, Squillante, Boopathy, Oseledets, and Daniel]{pmlr-v97-weng19a}
Lily Weng, Pin-Yu Chen, Lam Nguyen, Mark Squillante, Akhilan Boopathy, Ivan Oseledets, and Luca Daniel.
\newblock {PROVEN}: Verifying robustness of neural networks with a probabilistic approach.
\newblock In \emph{International Conference on Machine Learning}, 2019.

\bibitem[Weng et~al.(2020)Weng, Zhao, Liu, Chen, Lin, and Daniel]{Weng_Zhao_Liu_Chen_Lin_Daniel_2020}
Tsui-Wei Weng, Pu~Zhao, Sijia Liu, Pin-Yu Chen, Xue Lin, and Luca Daniel.
\newblock Towards certificated model robustness against weight perturbations.
\newblock In \emph{AAAI Conference on Artificial Intelligence}, 2020.

\bibitem[Wortsman et~al.(2024)Wortsman, Liu, Xiao, Everett, Alemi, Adlam, Co-Reyes, Gur, Kumar, Novak, Pennington, Sohl-Dickstein, Xu, Lee, Gilmer, and Kornblith]{wortsman2023smallscaleproxieslargescaletransformer}
Mitchell Wortsman, Peter~J. Liu, Lechao Xiao, Katie Everett, Alex Alemi, Ben Adlam, John~D. Co-Reyes, Izzeddin Gur, Abhishek Kumar, Roman Novak, Jeffrey Pennington, Jascha Sohl-Dickstein, Kelvin Xu, Jaehoon Lee, Justin Gilmer, and Simon Kornblith.
\newblock Small-scale proxies for large-scale transformer training instabilities.
\newblock In \emph{International Conference on Learning Representations}, 2024.

\bibitem[Yang et~al.(2024)Yang, Simon, and Bernstein]{yang2024spectralconditionfeaturelearning}
Greg Yang, James~B. Simon, and Jeremy Bernstein.
\newblock A spectral condition for feature learning.
\newblock \emph{arXiv:2310.17813}, 2024.

\bibitem[Yoshida and Miyato(2017)]{yoshida2017spectralnormregularizationimproving}
Yuichi Yoshida and Takeru Miyato.
\newblock Spectral norm regularization for improving the generalizability of deep learning.
\newblock \emph{arXiv:1705.10941}, 2017.

\bibitem[You(2025)]{You2025}
Jiacheng You.
\newblock Rapidly converging orthogonalizing {N}ewton-{S}chulz iteration, 2025.
\newblock URL \url{https://x.com/YouJiacheng/status/1893704552689303901}.

\bibitem[Zhao et~al.(2024)Zhao, Li, Gu, Zheng, Kölker, Wang, and Yuan]{zhao2024adapproxadaptiveapproximationadam}
Pengxiang Zhao, Ping Li, Yingjie Gu, Yi~Zheng, Stephan~Ludger Kölker, Zhefeng Wang, and Xiaoming Yuan.
\newblock Adapprox: Adaptive approximation in {A}dam optimization via randomized low-rank matrices.
\newblock \emph{arXiv:2403.14958}, 2024.

\end{thebibliography}

\clearpage

\appendix

\section{Coupling spectral cap to learning rate}
\label{app:spectral-softcap}

This section proves \cref{thm:spectral-soft-cap-bound}: spectral soft cap bounds weight norm. We will derive a strength parameter that couples to the learning rate, because otherwise using odd polynomial approximation---rather than the ideal map $\min(\sigma_\text{max}, \sigma)$---accumulates errors when the learning rate falls below the approximation gap. We will prove the theorem using an equilbrium analysis: solving for the fixed point of a contractive map.

To warm up, there is a special case in which weight decay provably bounds the weight norm during training. It happens when the norm of the update is bounded by the learning rate, $\norm{\Delta W} \leq \eta$. To see why, suppose weight decay is applied at every step of training and is linearly coupled to the learning rate as $\lambda \eta$ for some constant $\lambda > 0$. Subadditivity of norms guarantees $\norm{W + \Delta W} \leq \norm{W} + \norm{\Delta W} \leq \norm{W} + \eta$. The weights cannot increase further when the decay and the learning rate are in equilbrium: $\norm{W} \cdot (1 - \lambda \eta) + \eta = \norm{W}$. The equilibrium occurs at $\norm{W} = 1/\lambda$. Therefore $1/\lambda$ is the maximum weight norm possible under standard weight decay, if the update norm is bounded above by $\eta$. \citet{pethick2025trainingdeeplearningmodels} noted the same phenomenon.

To prove \cref{thm:spectral-soft-cap-bound}, we conduct a general equilibrium analysis to consider three effects: weight decay, optimizer step, and weight projection. The net effect of the three effects should decrease every singular value $\sigma \geq \sigma_\text{max}$. Let the weight decay $\lambda$ be coupled to the learning rate $\eta$. Suppose the weight update is constrained to have norm $\norm{\Delta W} \leq \eta$. Let $p(x)$ be an odd polynomial. Equilbrium occurs when \begin{equation}
    \label{eq:equilibrium}
    p(x \cdot (1 - \lambda \eta) + \eta) \leq x.
\end{equation} In words, apply weight decay, apply an optimizer step that will not increase singular values by more than $\eta$, and then apply the odd polynomial. If the singular value does not increase for all $x \in [0, \sigma_\text{max}]$, then the weight norm will never exceed $\sigma_\text{max}$.

Recall that $p_1(x) = x - \alpha x^3$ and $p_2(x) = x + \alpha x^3$.

When $p(x) = p_2(p_1(x))$, the equilibrium condition \cref{eq:equilibrium} becomes \begin{gather}
    f(\alpha) = \left(k-\alpha k^{3}\right)+\alpha\left(k-\alpha k^{3}\right)^{3} - \sigma_\text{max} \leq 0, \\
    \text{where } \; k = \sigma_\text{max} \cdot (1 - \lambda \eta) + \eta.
\end{gather} We can consider only $x = \sigma_\text{max}$ because, in this case, larger $x$ will decrease if a smaller $x$ decreases. Here $\alpha$ is the free variable. Finding the smallest $\alpha > 0$ amounts to solving the quartic polynomial \begin{equation}
    \label{eq:generalized-coupling}
    -k^9 \alpha^4 + 3k^7 \alpha^3 - 3k^5 \alpha^2 + k - \sigma_\text{max} = 0.
\end{equation} Any numerical solver can approximate $\alpha$. This is the $\alpha$ that makes $\sigma_\text{max}$ a fixed point under the overall training step. Connecting to the earlier equilibrium analysis, the special case $\alpha = 0$ is possible when already $k = \sigma_\text{max}$, or $\sigma_\text{max} = 1/\lambda$. While the coupling for spectral soft capping is not linear as in common implementations of AdamW, it succeeds at making tight weight norm bounds compatible with learning rate schedules that may tend to $0$. Initializing the weights near $\sigma_\text{max}$, rather than strictly less than it, suffices for the bound to hold throughout training in practice.

One limitation of automatic coupling is that it may be stronger than necessary, because it assumes updates align perfectly with the weights in the worst case. If the learning rate is scheduled to $0$, gradients may align less with the existing weights especially at the end, which can cause the weight norm to contract slightly.

\begin{figure}[t]
    \centering
    \includegraphics[width=0.4\textwidth]{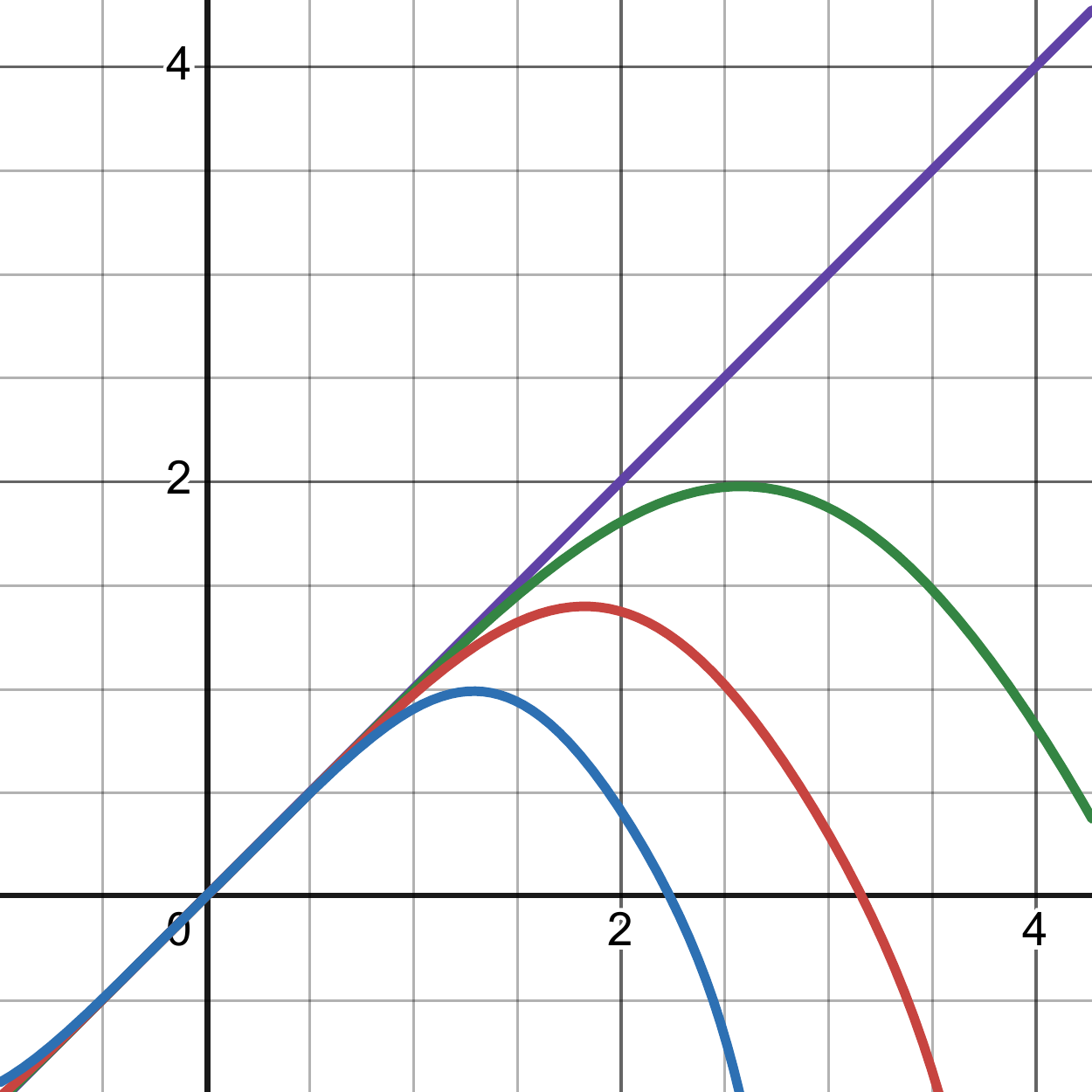}
    \caption{\textbf{Spectral soft cap} is a weight constraint method that applies the above odd polynomial to a weight matrix, which applies it to all singular values in parallel. First it applies $p_1(x) = x - \alpha x^3$, then it applies $p_2(x) = x + \alpha x^3$. The composition is depicted for $\alpha = 0.2$ (blue), $\alpha = 0.1$ (red), $\alpha = 0.05$ (green), $\alpha = 0$ (purple).}
    \label{fig:spectral-soft-cap-graph}
\end{figure}

\clearpage
\section{Spectral clipping}
\label{app:spectral-hardcap}

Spectral clipping is a generalization of spectral cap that puts an upper \textit{and} lower bound on the singular values. It maps $\sigma \mapsto \text{clip}(\sigma, \sigma_\text{min}, \sigma_\text{max})$, where $0 < \sigma_\text{min} < \sigma_\text{max}$ \citep{cesista2025spectralclipping, kexuefm-10795}.
\begin{definition}[Spectral clipping] Let $W \in \mathbb{R}^{m \times n}$ and $W = U \Sigma V^T$ be its singular value decomposition, where $\Sigma = (\sigma_1, \ldots, \sigma_{\min(m,n)})$ are the singular values of $W$. Then we define spectral clipping as the matrix function $\texttt{spectral\_clip}_{[\alpha, \beta]}: \mathbb{R}^{m \times n} \to \mathbb{R}^{m \times n}$ acting on the singular values of $W$,
\begin{equation}\texttt{spectral\_clip}_{[\alpha, \beta]}(W) = U \texttt{clip}_{[\alpha, \beta]}(\Sigma) V^T,\label{1}\end{equation}
where $\alpha \leq \beta$ and $\alpha, \beta \in \mathbb{R} \cup \{-\infty, \infty\}$ control the minimum and maximum attainable singular values. The clip function $\texttt{clip}_{[\alpha, \beta]}: \mathbb{R} \to \mathbb{R}$ is applied element-wise on the singular values of $W$,
\begin{equation}\texttt{clip}_{[\alpha, \beta]}(x) = \begin{cases}
\alpha & \text{if } x < \alpha \\
x & \text{if } \alpha \leq x \leq \beta \\
\beta & \text{if } \beta < x
\end{cases}.\end{equation}
\end{definition}

We could compute spectral clipping via SVD. However, SVD does not take full advantage of GPU tensor cores and typically requires casting the inputs to full precision, making it slow in practice. Following \citet{cesista2025spectralclipping}, we can instead spectrally clip via the matrix sign function used by Muon \citep{jordan2024muon} because of the following identity,
\begin{equation}\label{eq:spectralclipraw}\texttt{clip}_{[\alpha, \beta]}(x) = \frac{1}{2}[\alpha + \beta + (\alpha - x)\texttt{sign}(\alpha - x) - (\beta - x)\texttt{sign}(\beta - x)].\end{equation}
Thus,
\begin{align*}
    \texttt{spectral\_clip}_{[\alpha, \beta]}(W)
        &= U \texttt{clip}_{[\alpha, \beta]}(\Sigma) V^T\nonumber\\
        &= U \frac{1}{2}[(\alpha + \beta) I + (\alpha I - \Sigma)\texttt{sign}(\alpha I - \Sigma) - (\beta I - \Sigma)\texttt{sign}(\beta I - \Sigma)] V^T\nonumber\\
        &= \frac{1}{2} [(\alpha + \beta) UV^T + U (\alpha I - \Sigma ) \texttt{sign}(\alpha I - \Sigma) V^T - U (\beta I - \Sigma ) \texttt{sign}(\beta I - \Sigma) V^T]\nonumber\\
        &= \frac{1}{2} [(\alpha + \beta) UV^T\nonumber\\
        &\qquad+ U (\alpha I - \Sigma ) (V^TV) \texttt{sign}(\alpha I - \Sigma) (U^TU) V^T\nonumber\\
        &\qquad- U (\beta I - \Sigma ) (V^TV) \texttt{sign}(\beta I - \Sigma) (U^TU) V^T]\nonumber\\
        &= \frac{1}{2} [(\alpha + \beta) UV^T\nonumber\\
        &\qquad+ (\alpha UV^T - U\Sigma V^T) (V \texttt{sign}(\alpha I - \Sigma) U^T)(UV^T)\nonumber\\
        &\qquad- (\beta UV^T - U\Sigma V^T)  (V \texttt{sign}(\beta I - \Sigma) U^T)(UV^T)]\nonumber\\
        &= \frac{1}{2} [(\alpha + \beta) UV^T\nonumber\\
        &\qquad+ (\alpha UV^T - U\Sigma V^T) (U \texttt{sign}(\alpha I - \Sigma) V^T)^T(UV^T)\nonumber\\
        &\qquad- (\beta UV^T - U\Sigma V^T)  (U \texttt{sign}(\beta I - \Sigma) V^T)^T(UV^T)]\nonumber\\
        &= \frac{1}{2} [(\alpha + \beta) \texttt{msign}(W)\nonumber\\
        &\qquad+ (\alpha \cdot\texttt{msign}(W) - W) \texttt{msign}(\alpha \cdot\texttt{msign}(W) - W)^T\texttt{msign}(W)\nonumber\\
        &\qquad- (\beta  \cdot\texttt{msign}(W) - W) \texttt{msign}(\beta  \cdot\texttt{msign}(W) - W)^T\texttt{msign}(W)]\nonumber\\
    \texttt{spectral\_clip}_{[\alpha, \beta]}(W)
        &= \frac{1}{2} [(\alpha + \beta)I\nonumber\\
        &\qquad+ (\alpha \cdot\texttt{msign}(W) - W) \texttt{msign}(\alpha \cdot\texttt{msign}(W) - W)^T\nonumber\\
        &\qquad- (\beta  \cdot\texttt{msign}(W) - W) \texttt{msign}(\beta  \cdot\texttt{msign}(W) - W)^T\nonumber\\
        &\qquad]\texttt{msign}(W) \tag{\refstepcounter{equation}\theequation}.
\end{align*}

\subsection{Spectral hardcapping}

Spectral hardcapping is a special case of spectral clipping where $\alpha \leq 0$. And since singular values are always nonnegative, setting $\alpha = -\beta$ simplifies \cref{eq:spectralclipraw} to
\begin{equation}
    \texttt{clip}_{[-\beta, \beta]}(x) = \frac{1}{2}[\beta + x - (\beta - x)\texttt{sign}(\beta - x).
\end{equation}
We can now construct a spectral hardcapping formula in terms of the well-known \texttt{msign} function:
\begin{align*}
    \texttt{spectral\_hardcap}_{\beta}(W)
        &= U \texttt{clip}_{[-\beta, \beta]}(\Sigma) V^T \\
        &= U \frac{1}{2}[\beta I + \Sigma - (\beta I- \Sigma)\texttt{sign}(\beta I - \Sigma)] V^T \\
        &= \frac{1}{2}[\beta UV^T + U\Sigma V^T - U (\beta I- \Sigma)(V^T 
        V)\texttt{sign}(\beta I - \Sigma) (U^T U) V^T \\
        &= \frac{1}{2}[\beta UV^T + U\Sigma V^T - (U (\beta I- \Sigma)V^T) (U\texttt{sign}(\beta I - \Sigma) V^T)^T (U V^T) \\
        &= \frac{1}{2}[\beta \texttt{msign}(W) + W - (\beta\texttt{msign}(W) - W) \texttt{msign}(\beta\texttt{msign}(W) - W)^T \texttt{msign}(W)] \\
    \texttt{spectral\_hardcap}_{\beta}(W)
        &= \frac{1}{2}[\beta\texttt{msign}(W) + W - \texttt{msign}(\beta I - \texttt{msign}(W)W^T) (\beta\texttt{msign}(W) - W)].
\end{align*}
The last equality follows from the transpose equivariance and unitary multiplication equivariance of odd analytic matrix functions acting on singular values.

\subsection{spectrally clipped weight decay}

As a further extension, we can use spectral hardcapping to construct a spectrally clipped weight decay. Unlike standard weight decay, spectrally clipped weight decay only applies the decay term to singular values larger than a threshold $\beta$, chosen \textit{a priori}:
\begin{align}
\texttt{clipped\_weight\_decay}_{\lambda,\beta}(\Sigma)
    &= \begin{cases}
        \Sigma & \texttt{if } \Sigma \leq \beta\\
        (1-\lambda)\Sigma + \lambda\beta & \texttt{if } \Sigma > \beta \\
    \end{cases}\\
    &= (1-\lambda)\Sigma + \lambda\cdot\texttt{clip}_{[0, \beta]}(\Sigma) \nonumber.
\end{align}
Thus,
\begin{align}
    \texttt{spectral\_clipped\_weight\_decay}_{\lambda,\beta}(W)
        &= U \texttt{clipped\_weight\_decay}_{\lambda,\beta}(\Sigma) V^T\\
        &= (1-\lambda) W + \lambda\cdot\texttt{spectral\_hardcap}_\beta(W)
\end{align}

Following the argument made in \cref{app:spectral-softcap}, we can derive the equilibrium point of spectrally clipped weight decay as follows.
\begin{proposition}[Equilibrium point of spectrally clipped weight decay]
    Let $\eta \in (0, \infty)$ be the learning rate, $\lambda \in (0, 1]$ be the decay term, and $\beta \in (0, \infty)$ be the singular value threshold above which we start applying the decay term. Additionally, suppose that the weight updates are constrained to have norm $||\Delta W|| \leq \eta$ such as with the Muon optimizer. Then spectrally clipped weight decay has an equilibrium point, \begin{equation} \sigma_{\text{eq}} = \beta + \frac{1-\lambda}{\lambda}\eta,\end{equation}
    toward which it ``pulls'' the spectral norm of the weights.
\end{proposition}

\begin{proof} An update step yields
    $$W_{t+1} = \texttt{spectral\_clipped\_weight\_decay}_{\lambda,\beta}(W_t + \Delta W_t).$$
    The subadditivity of norms tells us $||W_t + \Delta W_t|| \leq ||W_t|| + ||\Delta W_t|| \leq ||W_t|| + \eta$. Thus, we can bound the spectral norm of the weights after every update step as 
\begin{align*}
    \sigma'_{\max} &\leq \texttt{clipped\_weight\_decay}_{\lambda,\beta}(\sigma_{\max} + \eta)\\
    \sigma'_{\max} &\leq \begin{cases}
        \sigma_{\max} + \eta & \texttt{if } \sigma_{\max} + \eta \leq \beta\\
        (1-\lambda)(\sigma_{\max} + \eta) + \lambda\beta & \texttt{if } \sigma_{\max} + \eta > \beta
    \end{cases}
\end{align*}
    Equality is achieved at $\sigma_{\text{eq}}$, where
\begin{align*}
    \sigma_{\text{eq}} &= \begin{cases}
        \sigma_{\text{eq}} + \eta & \texttt{if } \sigma_{\text{eq}} + \eta \leq \beta\\
        (1-\lambda)(\sigma_{\text{eq}} + \eta) + \lambda\beta & \texttt{if } \sigma_{\text{eq}} + \eta > \beta
    \end{cases}\\
    \sigma_{\text{eq}} &= (1-\lambda)\sigma_{\text{eq}} + (1-\lambda)\eta + \lambda\beta\\
    \sigma_{\text{eq}} &= \beta + \frac{1-\lambda}{\lambda}\eta
\end{align*}
    Note that singular values larger than $\sigma_{\text{eq}}$ decrease after every update step,
\begin{align*}
    \text{update}(\sigma_{\text{eq}} + \epsilon) &= (1-\lambda)(\sigma_{\text{eq}} + \eta + \epsilon) + \lambda\beta\\
    &= \underbrace{(1-\lambda)(\sigma_{\text{eq}} + \eta) + \lambda\beta}_{\sigma_{\text{eq}}} + (1-\lambda)\epsilon\\
    \text{update}(\sigma_{\text{eq}} + \epsilon) &< \sigma_{\text{eq}} + \epsilon,
\end{align*}
    since $1-\lambda < 1$, while singular values smaller than $\sigma_{\text{eq}}$ increase,
\begin{align*}
    \text{update}(\sigma_{\text{eq}} - \epsilon) &= (1-\lambda)(\sigma_{\text{eq}} + \eta - \epsilon) + \lambda\beta\\
    &= \sigma_{\text{eq}} - (1-\lambda)\epsilon\\
    \text{update}(\sigma_{\text{eq}} - \epsilon) &> \sigma_{\text{eq}} - \epsilon.
\end{align*}
    Hence $\sigma_{\text{eq}}$ is an equilibrium point.
\end{proof}

A potentially useful property of spectrally clipped weight decay is that its equilibrium point approaches $\beta$ as learning rate is decayed to zero during training, independent of the chosen initial learning rate and decay term:
$$\sigma^*_{\text{eq}} = \lim_{\eta \to 0} \beta + \frac{1-\lambda}{\lambda}\eta = \beta.$$
This property may enable tighter final norm bounds without requiring as aggressive of a decay.

\clearpage
\section{Proving an upper bound on the Lipschitz constant of a transformer}
\label{app:lipschitz-attention}

We elaborate on the algorithm sketched in \cref{subsec:lipschitz-constant-of-transformer} and prove a Lipschitz bound on attention. Our Lipschitz bounds are with respect to the max $\rms$ norm over token positions, denoted $\norm{\cdot}_{\infty\rms}$.

Recall the two primary ways Lipschitz constants $L_f$ and $L_g$ of two functions $f$ and $g$ interact:
\begin{itemize}
    \item \textbf{Adding:} $f + g$ has Lipschitz constant at most $L_f + L_g$.
    \item \textbf{Composing:} $f \circ g$ has Lipschitz constant at most $L_f \cdot L_g$.
\end{itemize}

\textbf{Step 1: Residual connections.} Suppose that, before reaching a certain residual connection, a transformer maps input data $x$ to $f(x)$ with Lipschitz constant $L$. Suppose the transformer has $2N$ residual connections. Let $\alpha = \tfrac{1}{2N}$. The residual connection acts on $f(x)$ as \begin{equation}
    \left[ (1 - \alpha) \cdot \mathsf{identity} + \alpha \cdot \mathsf{block} \right](f(x)).
\end{equation} After the residual connection, the Lipschitz constant composes and adds to become at most \begin{equation}
    (1 - \alpha) \cdot L + \alpha \cdot L \cdot L_\mathsf{block}.
\end{equation} Applying this formula sequentially upper bounds the Lipschitz constant of a transformer layer by layer. We now determine $L_\mathsf{block}$ for an MLP and attention block in terms of their weight norms.

\textbf{Step 2: MLP.} Our MLP composes $W_\mathsf{out} \circ (\mathsf{GeLU} / 1.1289) \circ W_\mathsf{in}$. The Lipschitz constants of the two weight matrices are their norms $\norm{W_\mathsf{out}}_{\rms\to\rms}$ and $\norm{W_\mathsf{in}}_{\rms\to\rms}$, while $\mathsf{GeLU} / 1.1289$ has Lipschitz constant $1$ because we divide by the maximum derivative of $\mathsf{GeLU}$. Overall, the Lipschitz bound for an MLP block is $L_\mathsf{MLP} \leq \norm{W_\mathsf{out}}_{\rms\to\rms} \norm{W_\mathsf{in}}_{\rms\to\rms} / 1.1289$.

\textbf{Step 3: Attention.} Let $\ell$ denote the token dimension. Let the queries, keys, and values be denoted by $(q, k, v) \in \R^{\ell \times d_Q} \times \R^{\ell \times d_Q} \times \R^{\ell \times d_V}$. Our attention block composes \begin{equation}
    \tfrac{1}{3} W_\mathsf{out} \circ F,
\end{equation} where function attention is denoted by $F = \softmax\left(\tfrac{1}{d_Q} qk^\top + M\right)v$ for some mask $M$. As a consequence of the following theorem, if every attention input is unit norm, then functional attention is $1$-Lipschitz. This property is what motivates scaling functional attention by $\tfrac{1}{d_Q}$ rather than $\smash{\tfrac{1}{\sqrt{d_Q}}}$ inside the softmax. Composing functional attention with its input, the tuple $(q, k, v)$, increases its sensitivity to $3$; we scale by $\tfrac13$ to make attention as a whole have unit sensitivity. However, functional attention is no longer $1$-Lipschitz if its inputs are not unit norm. Recall that the shorthand notation $\norm{x}_{\infty\rms}$ is the max RMS norm of a $d$-dimensional activation over $l$ tokens, $x \in \R^{\ell \times d}$.

\begin{theorem}[Lipschitz bound on functional attention]
    \label{thm:function-attention-is-lipschitz}
    Let $\diamond$ denote tensor contraction. Given any perturbations $\Delta q, \Delta k, \Delta v$ to the queries, keys, and values, functional attention satisfies \begin{equation}
    \norm{\nabla F(q, k, v) \diamond (\Delta q, \Delta k, \Delta v)} \leq \max(1, \norm{v} \max(\norm{q}, \norm{k})) \norm{(\Delta q, \Delta k, \Delta v)},
\end{equation} where the norm is $\norm{\cdot}_{\infty \rms} : \R^{\ell \times d} \to \R$, the max-over-tokens $\rms$ norm of the embedding vector, and $\norm{(\Delta q, \Delta k, \Delta v)} := \norm{\Delta q} + \norm{\Delta k} + \norm{\Delta v}$. That is, functional attention has Lipschitz bound $\max(1, \norm{v} \max(\norm{q}, \norm{k}))$.
\end{theorem}

\begin{proof}
    The argument mirrors the proof of Proposition 7 from the modular norm paper \citep{large2024}. We write the attention matrix as $A = \softmax\left(\tfrac{1}{d_Q} qk^\top + M\right)$. Its derivative is $\Delta A = \nabla_{(q,k)} \softmax\left(\tfrac{1}{d_Q} qk^\top + M\right) \diamond (\Delta q, \Delta k)$. The derivative of $F$ splits into two terms, \begin{equation}
        \label{eq:attention-lipschitz}
        \nabla F(q, k, v) \diamond (\Delta q, \Delta k, \Delta v) = A (\Delta v) + (\Delta A) v.
    \end{equation} We call the maximum $\ell_1$ norm of the rows of a matrix its $L^\infty$ operator norm, which comes into play by observing that $\norm{Ax}_{\infty\rms} \leq \norm{A}_{\infty-\mathsf{op}} \norm{x}_{\infty\rms}$. For the first term, note that $\norm{A}_{\infty-\mathsf{op}} = 1$ because softmax ensures the row-wise sum is always $1$. For the second term, \citet{large2024} in Equation E.58 show that \begin{equation}
        \norm{\Delta A}_{\infty-\mathsf{op}} \leq \norm{\Delta q}_{\infty\rms} \norm{k}_{\infty\rms} + \norm{q}_{\infty\rms} \norm{\Delta k}_{\infty\rms}.
    \end{equation}
    Thus, writing $\norm{\cdot}$ as shorthand for $\norm{\cdot}_{\infty \rms}$,
    \begin{align*}
        \norm{\nabla F(q, k, v) \diamond (\Delta q, \Delta k, \Delta v)}
            &= \norm{A (\Delta v)} + \norm{(\Delta A) v} \\
            &\leq \cancel{\norm{A}_{\infty-\mathsf{op}}} \norm{\Delta v} + \norm{\Delta A}_{\infty-\mathsf{op}} \norm{v} \\
            &\leq \norm{\Delta v} + \norm{v}\norm{k} \norm{\Delta q} + \norm{v}\norm{q} \norm{\Delta k} \\
            &\leq \norm{\Delta v} + \norm{v}\max(\norm{q}, \norm{k})(\norm{\Delta q} + \norm{\Delta k}) \\
            &\leq \max(1, \norm{v}\max(\norm{q}, \norm{k})) (\norm{\Delta q} + \norm{\Delta k} + \norm{\Delta v})
    \end{align*}
    Hence, $\norm{\nabla F(q, k, v) \diamond (\Delta q, \Delta k, \Delta v)} \leq \max(1, \norm{v}\max(\norm{q}, \norm{k})) \norm{(\Delta q, \Delta k, \Delta v)}$ as claimed.
\end{proof}

More generally for attention layers with attention scale $s_{\text{attn}}$ not necessarily equal to $\tfrac{1}{d_Q}$, we can absorb the extra factor to the query weight $W_Q$ \textit{and} key weight $W_K$, such that \begin{equation} \tilde{F} = \softmax\left(s_{\text{attn}} qk^T+M\right)v = \softmax\left(\frac{1}{d_Q} \tilde{q}\tilde{k}^T+M\right)v, \end{equation} where $\tilde{q} = \sqrt{s_{\text{attn}}d_Q} q$ and $\tilde{k} = \sqrt{s_{\text{attn}}d_Q} k$. The Lipschitz bound then is, \begin{equation} \max(1, \norm{v} \max(\norm{\tilde{q}}, \norm{\tilde{k}})) = \sqrt{s_{\text{attn}}d_Q} \max(1, \norm{v} \max(\norm{q}, \norm{k})). \end{equation}

\textbf{Step 4. Activation norm bounds.} To apply the theorem, we now bound the input norm to attention. To do so we will track the maximum $\rms$ norm of activations everywhere in the network. We do not use layer norm and therefore cannot reset activation norms to 1. Let $x_0, \dots, x_{2N}$ denote all the activations, from the initial embedding $x_0$ through to the $N$ alternating attention and MLP blocks acting via residual connections. Suppose the embedding layer maps tokens to have $\rms$ norm at most $1$, or $\norm{x_0}_{\infty\rms} \leq 1$. Attention and MLP increase the norm as follows:
\begin{itemize}
    \item Attention computes $W_\mathsf{out} \circ (V, A)$ for some attention matrix $A$, where $(V, A)$ is shorthand for functional attention. By definition $V$ cannot increase the $\rms$ norm of the embedding $x_i$ at any token by more than its $\rms\to\rms$ operator norm, meaning $\norm{V x_i}_{\infty\rms} \leq \norm{V}_{\rms\to\rms} \norm{x_i}_{\infty\rms}$. The same bound applies to $(V, A)x_i$ by subadditivity of norms, since entries of the attention matrix $A$ sum to $1$ in the token dimension. Therefore attention can increase the activation norm by \begin{equation}
        \norm{(W_\mathsf{out} \circ (V, A))x_i}_{\infty\rms} \leq \norm{W_\mathsf{out}}_{\rms\to\rms} \norm{V}_{\rms\to\rms} \norm{x_i}_{\infty\rms}.
    \end{equation} In words, multiply the weight norms of $W_\mathsf{out}$ and $V$ to get the maximum increase.
    \item The MLP computes $W_\mathsf{out} \circ (\mathsf{GeLU} / 1.1289) \circ W_\mathsf{in}$. Therefore the MLP can increase activation norm by $\norm{W_\mathsf{out}}_{\rms\to\rms} \norm{W_\mathsf{in}}_{\rms\to\rms} / 1.1289$, since $|\mathsf{GeLU}(x)| \leq |x|$ for all $x \in \R$.
    \item The residual connection acts like \begin{equation}
        \norm{(1 - \alpha) \cdot x_i + \alpha \cdot \mathsf{block}(x_i)}_{\infty\rms} \leq (1 - \alpha) \norm{x_i}_{\infty \rms} + \alpha \norm{\mathsf{block}(x_i)}_{\infty\rms}.
    \end{equation}
\end{itemize}

\textbf{Algorithm to compute Lipschitz bound.} Therefore, given the weight norms of all matrices in a transformer, we use the preceding results to compute its Lipschitz bound in two steps. First, we upper bound the activation norm everywhere in the network using Step 4. Second, we upper bound the Lipschitz constant using Steps 1-3. The Lipschitz bound after the final layer is what we refer to as the transformer's Lipschitz upper bound.

\clearpage
\section{Implementing LipsFormer and bounding its Lipschitz constant}

To turn our enforced norm training into LipsFormer \citep{qi2023lipsformer}, we make the following changes:

\begin{enumerate}
    \item Remove spectral soft cap and embed projections.
    \item Use CenterNorm: mean subtraction with learnable entrywise scale and bias.
    \item Use scaled-head cosine attention with $\epsilon = 10^{-6}$, $\tau = 12$, $\nu = 1$. Notably, the \href{https://github.com/IDEA-Research/LipsFormer/blob/main/models/lipsformer_swin.py#L205-L206}{official implementation} of LipsFormer uses $\epsilon = 0$. According to their Theorem 1, this choice may make a finite Lipschitz bound impossible. We set $\epsilon > 0$ to fix the issue.
    \item Heuristically scale down attention output by $1/n_\text{heads}$ to match their implementation.
    \item Insert residual connections with learnable strength $\alpha$, initialized to $1/n_\text{residual\_connections}$.
    \item Xavier normal initialize linear layers, then apply spectral normalization $W \mapsto W / \norm{W}_\ast$.
    \item Include drop path: every residual connection is skipped with $p=0.5$ and, if taken, is scaled up by $1/(1-p)$, matching their official implementation which uses \verb|nn.Dropout|.
    \item Use weight decay 0.1, matching their implementation (not applied to scalar parameters).
    \item Use the Muon optimizer to give LipsFormer the fairest comparison, copying hyperparameters from our run. We tested training with AdamW for all parameters, an exact replication, but found performance degraded sigificantly: after 1770 steps, validation loss was 4.86 (compared to 3.61) and validation accuracy was 0.227 (compared to 0.301).
    \item For non-weight-matrix parameters, use Adam hyperparameters $\eta = 0.001$, $\beta_1 = 0.9$, $\beta_2 = 0.999$, $\epsilon = 10^{-8}$ to match their implementation.
    \item Use cosine learning rate schedule with decay to 0 to match their implementation.
\end{enumerate}

\textbf{Bounding the Lipschitz constant of LipsFormer.} In \cref{tab:nanogptspeedrun}, we report that our trained implementation of LipsFormer has a Lipschitz upper bound of $10^{130}$. To calculate this value, we use the final weight norms of the MLP and attention blocks to bound the Lipschitz constant of each residual block, relying on LipsFormer's Theorem 1:
\[
    \mathsf{Lip}(\text{SCSA})_2 \leq 2N(N-1)\nu\tau\epsilon^{-\tfrac12}\norm{W^K}_2 + 2(N-1)\nu\tau\epsilon^{-\tfrac12}\norm{W^Q}_2 + 2N\nu\epsilon^{-\tfrac12}\norm{W^V}_2.
\]
Using $N = 128$ (head dimension), $\tau = 12$, $\nu = 1$, and empirical weight norms, we calculate the Lipschitz bound for every layer. We use the maximum entry of the learned residual strength $\alpha$, which is an entrywise multiplication, to convert the layerwise bounds into a final bound \[
    \mathsf{Lip}(F) \leq \prod_{s=1}^S \prod_{m=1}^S (1 + \alpha_{s,m} \mathsf{Lip}(f_{s,m})),
\] which we take from their Equation 19. Alpha has typical maximum entries around $0.5$ for attention connections and $0.15$ for MLP connections. With $\epsilon = 10^{-6}$, we compute a final Lipschitz bound of $1.97 \times 10^{129}$.

\clearpage
\section{Experimental details}
\label{app:experimental-details}

This section gives experimental details for all results in the paper. The three categories of experiments we run are MLP training, Shakespeare transformer training, and NanoGPT speedrun training.

\textbf{Datasets.}
\begin{itemize}
    \item For MLP training we use the CIFAR-10 dataset \cite{krizhevsky2009cifar} with the standard train and test splits and no data augmentation. We do not shuffle the order of batches.
    \item For Shakespeare transformer training we use Karpathy's 1M character-level dataset with standard training and validation splits \citep{karpathy_nanoGPT}. We shuffle the order of batches.
    \item For NanoGPT speedrun transformer training we use the FineWeb10B dataset \citep{penedo2024finewebdatasetsdecantingweb} loaded in the standard order. We use the same validation split as the modded NanoGPT speedrun benchmark \citep{modded_nanogpt_2024}.
\end{itemize}

\textbf{Compute requirement.} All our experiments can run on a V100, A100, or H100 GPU in less than 5 minutes, except the NanoGPT speedrun transformer which requires 8xH100 and runs in 5-10 minutes.

\textbf{Modula library.} For MLP and Shakespeare experiments, we use JAX \citep{jax2018github} on top of the Modula library \citep{large2024,modula-docs}. We implement our own model components. Our AdamW implementation does not include bias correction, although the discrepancy decays rapidly after aronud 20 steps because we use $\beta_1 = 0.9, \beta_2 = 0.95$ in all experiments except one, not reported, in which we determine that this is a good setting for the momentum EMAs.

\textbf{MLP experiments.} All MLPs we train are width 256 and depth 3 (i.e., one hidden layer) with $\mathsf{ReLU}$ activations and no bias on data from CIFAR-10. We use batch size 512 and a linear learning rate schedule that decays to 0 in all experiments. Modula's mass calculation causes the effective learning rate to be scaled by $1/3$. We train for 50 epochs except in one case, when we train for 20 epochs for the models in \cref{fig:accuracy_vs_epsilon}. We zero-initialize the final layer. We train all models in float32 precision and run the weight constrain methods in float32 precision. We experimented with lower precision and found comparable metrics across the board for bfloat16 training. We set seed 0 and store all hyperparameters and log information to enhance reproducibility.

\textbf{Shakespeare experiments.} All transformers we train for Shakespeare are width 256 with 3 blocks (attention + MLP), no bias, and four attention heads. The out projection in each attention and MLP block is initialized to zero. We use sequence length 256 and batch size 64 to match the baseline from \citep{karpathy_nanoGPT}, except we train for 2000 steps while Karpathy trains for 5000 steps. We set Modula's blocks mass parameter to 32 to cause 95\% of the feature learning to occur in the transformer blocks. We determined this ratio by sweeping the blocks mass, which controls the ratio of learning rate between the two embedding layers and the transformer blocks. Training with Muon means applying Muon to all linear layer weight matrices (including the final logit head) but normalizing the columns of embedding gradient, as suggested by the $\ell_1 \to \rms$ duality map \citep{bernstein2024modulardualitydeeplearning}. We were concerned that rare tokens may cause the momentum buffer to dualize columns to full strength updates for hundreds of steps until the column decays to exactly zero, so we tested whether capping the maximum inflation factor for the embedding column normalization could help. We tested maximum factors in the set $\{1, 4, 16, \dots, 65536\}$ across 8 seeds and found no significant difference. We choose to maximally multiply each column by 16 during the dualization step. Finally, we found that to train to the validation losses reported we had to use a trick: we decayed the learning rate by a factor of 1/2 per residual layer, causing later layers to train more than earlier layers. This change is implemented by setting the sensitivity of the $\mathsf{Mul}$ module in Modula to 1. We do not know why this trick is necessary.

\cref{fig:pareto} sweeps over the following hyperparameters:
\begin{itemize}
    \item MLPs on CIFAR-10: we test the following combinations of optimzer and weight constraint method: (AdamW, weight decay), (AdamW, spectral weight decay), (AdamW, spectral normalization), (AdamW, Stiefel manifold projection), (AdamW, spectral hammer),  (Muon, weight decay), (Muon, spectral weight decay), (Muon, spectral normalization), (Muon, stiefel manifold projection), (Muon, spectral hammer), (Muon, spectral soft cap), (Muon, spectral hard cap). For AdamW, we vary the weight decay and spectral weight decay parameters with 10 points in log-space from $10^{-2}$ to $10^{0}$. For Muon we vary the weight decay parameter with 10 points in log-space from $10^{-3}$ to $10^{0}$ and the spectral weight decay parameter with 10 points in log-space from $10^{-2}$ to $10^{0}$. For AdamW with spectral normalization, Stiefel manifold projection, and spectral hammer, we vary the maximum weight norm in the set $\sigma_\text{max} \in \{2, 3, 4, 5, 6, 7, 8\}$. For Muon with spectral normalization, Stiefel manifold projection, spectral soft cap, and spectral hard cap, we vary the maximum weight norm in the set $\sigma_\text{max} \in \{4, 5, 6, 7, 8, 9, 10\}$. For Muon with spectral hammer we vary the maximum weight norm in the set $\sigma_\text{max} \in \{1, 2, 3, \ldots , 10\}$. For AdamW with all methods we sweep 16 learning rates in log-space between $10^{-5}$ and $10^{-0.5}$. For Muon we sweep 16 learning rates in log-space between $10^{-2}$ and $10^{1}$ for all methods except for spectral hammer, where we use these learning rates for $\sigma_\text{max} \in \{4, 5, 6, 7, 8, 9, 10\}$, and use 16 learning rates in log-space between $10^{-3}$ and $10^{0}$ for  $\sigma_\text{max} \in \{1, 2, 3\}$. Overall, this results in 1,610 total combinations, 682 with AdamW and 928 with Muon.
    \item Transformers on Shakespeare: we test the following combinations of optimizer and weight constraint method: (AdamW, weight decay), (AdamW, spectral normalize), (AdamW, spectral hammer), (Muon, weight decay), (Muon, spectral normalize), (Muon, spectral soft cap). For spectral normalize, spectral hammer, and spectral soft cap, we vary the maximum weight norm in the set $\sigma_\text{max} \in \{1.0, 1.2, \dots, 2.8, 3.0\}$. For the baseline, we vary weight decay in the set $\lambda \in \{2/3, 0.5, 0.4, 0.3, 0.2, 0.1, 0.05, 0.03, 0.01, 0\}$. For AdamW we sweep 16 learning rates between $10^{-4.5}$ and $10^{-1.5}$. For Muon, we sweep 12 learning rates between $10^{-1.5}$ and $10^{1.5}$. We ran tests before to find ranges that cover the optimal learning rate.
\end{itemize}

\cref{fig:accuracy_vs_epsilon} reports adversarial examples and dataset-wide statistics from two models trained for 20 epochs. The AdamW model is trained with learning rate $8.1 \times 10^{-3}$ and weight decay $\lambda = 0.1$. The Muon model is trained with learning rate $2.3 \times 10^{-1}$ and weight decay $\lambda = 0$, using the spectral soft cap method with a weight constraint of $\sigma_\text{max} = 3$.

The left panel of \cref{fig:fullmuoncomp} visualizes the same data from the experiment for \cref{fig:pareto}, but focuses only on MLPs trained with Muon on CIFAR-10. The middle and right panels use the Muon optimizer, with the following tuples of (weight constraint method, maximum singular value, weight decay, spectral weight decay, learning rate): (weight decay, N/A, 0.1, 0, 1.585), (spectral weight decay, N/A, 0, 0.05, 0.157), (spectral normalization, 6, 0, 0, 1.0), (Stiefel manifold projection, 5, 0, 0, 1.0), (spectral hammer, 4, 0, 0, 0.398), (spectral soft cap, 6, 0, 0, 0.398), (spectral hard cap, 5, 0, 0, 0.631).

\textbf{NanoGPT experiments.} Following the Modded-NanoGPT speedrun standard \citep{modded_nanogpt_2024}, our training runs print log files with the full source code required to reproduce the results. We briefly summarize the changes we made to convert the NanoGPT speedrun record (as of May 2025) into our method:
\begin{itemize}
    \item Every step, $\rms$ normalize the embedding columns.
    \item Initialize all linear layer weight matrices to be orthogonal.
    \item Reparameterize residual connections according to \cref{eq:breaking-barrier-convex-combination}: $\tfrac{L-1}{L} x + \tfrac{1}{L} \mathsf{block}(x)$ residual connections, where $L = 24$ is the number of residual connections.
    \item Reparameterize attention according to \cref{eq:breaking-barrier-softmax-1-d}: $\tfrac{1}{3}$ overall scale on the attention output and $1/d_\text{head}$ scale inside the softmax.
    \item Every step, apply spectral soft cap (or spectral normalize) to every linear layer weight matrix based on a prespecified maximum desired weight norm $\sigma_\text{max}$.
    \item Use different orthogonalization coefficients that at most inflate a singular value to 1.14502. Therefore, the maximum update norm we pass to the strength parameter solver for learning rate coupling in spectral soft cap is $\eta \cdot 1.14502 \cdot 1.05$ with an extra factor of $1.05$ to be safe around numerical precision errors. The iteration is derived by modifying the method in \citep{cesista2025muonoptcoeffs}.
    \item Remove U-net structure.
    \item Use $\mathsf{GeLU} / 1.1289$ instead of $\mathsf{ReLU}^2$.
    \item Switch the dimension scaling in Muon to be $\sqrt{\mathsf{fan\_{out}} / \mathsf{fan\_{in}}}$ instead of $\max(1, \sqrt{\mathsf{fan\_{out}} / \mathsf{fan\_{in}}})$.
    \item Remove $\rms$ normalization: the model is now Lipschitz continuous.
    \item Add back the 7th attention layer (which was removed in the speedrun).
    \item Weight projections are run in bfloat16 (which we found to slightly improve performance). Spectral normalization uses 2 iterations, meaning that weight norms can exceed the specified maximum $\sigma_\text{max}$ due to approximation error; in practice weights with norms enforced by spectral normalization exceed the specified maximum by around 10\%.
\end{itemize}

\end{document}